\documentclass[10pt,twocolumn,twoside]{IEEEtran}
\usepackage{cite}
\usepackage[utf8]{inputenc} 
\usepackage[T1]{fontenc}    
\usepackage{hyperref}       
\usepackage{url}            
\usepackage{booktabs}       
\usepackage{amsfonts}       
\usepackage{nicefrac}       
\usepackage{microtype}      
\usepackage{doi}
\usepackage{amsmath,amssymb,amsfonts}
\usepackage{algorithmic}
\usepackage{graphicx}
\usepackage{textcomp}
\usepackage{xcolor}
\usepackage{bbm}
\usepackage{amsthm}
\usepackage{setspace}
\usepackage{physics}
\usepackage{algorithm}
\usepackage{mathtools}
\usepackage{diagbox}
\usepackage{caption}
\usepackage{subcaption}
\usepackage{physics}
\usepackage{multirow}
\usepackage{epstopdf}
\usepackage{tabularx}
\makeatletter
\newcommand{\distas}[1]{\mathbin{\overset{#1}{\kern\z@\sim}}}%
\newsavebox{\mybox}\newsavebox{\mysim}
\newcommand{\distras}[1]{%
  \savebox{\mybox}{\hbox{\kern3pt$\scriptstyle#1$\kern3pt}}%
  \savebox{\mysim}{\hbox{$\sim$}}%
  \mathbin{\overset{#1}{\kern\z@\resizebox{\wd\mybox}{\ht\mysim}{$\sim$}}}%
}
\makeatother

\makeatletter
\newcommand*\rel@kern[1]{\kern#1\dimexpr\macc@kerna}
\newcommand*\widebar[1]{%
  \begingroup
  \def\mathaccent##1##2{%
    \rel@kern{0.8}%
    \overline{\rel@kern{-0.8}\macc@nucleus\rel@kern{0.2}}%
    \rel@kern{-0.2}%
  }%
  \macc@depth\@ne
  \let\math@bgroup\@empty \let\math@egroup\macc@set@skewchar
  \mathsurround\z@ \frozen@everymath{\mathgroup\macc@group\relax}%
  \macc@set@skewchar\relax
  \let\mathaccentV\macc@nested@a
  \macc@nested@a\relax111{#1}%
  \endgroup
}
\makeatother
\newtheorem{definition}{Definition}
\newtheorem{assumption}{Assumption}
\newtheorem{theorem}{Theorem}
\newtheorem{lemma}{Lemma}
\newtheorem{remark}{Remark}
\newtheorem{fact}[definition]{Fact}
\newtheorem{corollary}{Corollary}

\title{Distributed Consensus Algorithm for Decision-Making in Multi-agent Multi-armed Bandit}
\author{Xiaotong Cheng and Setareh Maghsudi\thanks{X. Cheng is with the Department of Computer Science, University of Tübingen, 72074 Tübingen, Germany (email:xiaotong.cheng@uni-tuebingen.de). S. Maghsudi is with the Department of Computer Science, University of Tübingen, 72074 Tübingen, Germany and with the Fraunhofer Heinrich Hertz Institute, 10587 Berlin, Germany.}}
\begin{document}
\maketitle
\begin{abstract}
We study a structured multi-agent multi-armed bandit (MAMAB) problem in a dynamic environment. A graph reflects the information-sharing structure among agents, and the arms' reward distributions are piecewise-stationary with several unknown change points. The agents face the identical piecewise-stationary MAB problem. The goal is to develop a decision-making policy for the agents that minimizes the regret, which is the expected total loss of not playing the optimal arm at each time step. Our proposed solution, Restarted Bayesian Online Change Point Detection in Cooperative Upper Confidence Bound Algorithm (RBO-Coop-UCB), involves an efficient multi-agent UCB algorithm as its core enhanced with a Bayesian change point detector. We also develop a simple restart decision cooperation that improves decision-making. Theoretically, we establish that the expected group regret of RBO-Coop-UCB is upper bounded by $\mathcal{O}(KNM\log T + K\sqrt{MT\log T})$, where $K$ is the number of agents, $M$ is the number of arms, and $T$ is the number of time steps. Numerical experiments on synthetic and real-world datasets demonstrate that our proposed method outperforms the state-of-the-art algorithms.
\end{abstract}
\section{Introduction}
Multi-armed bandit (MAB) is a fundamental problem in online learning and sequential decision-making. In the classical setting, a single agent adaptively selects one among a finite set of arms (actions) based on past observations and receives a reward accordingly. The agent repeats this process over a finite time horizon to maximize its cumulative reward. The MAB framework has been applied in several areas such as computational advertisement \cite{tang2013automatic}, wireless communications \cite{maghsudi2014joint} and online recommendation \cite{li2016collaborative}.

To date, most research on the MAB problem focuses on single-agent policies, neglecting the social components of the applications of the MAB framework; nevertheless,  the ever-increasing importance of networked systems and large-scale information networks motivates the investigation of the MAB problem with multiple agents \cite{landgren2021distributed}. For example, the users targeted by a recommender system might be a part of a social network \cite{cesa2013gang}. In that case, the network structure is a substantial source of information, which, if taken advantage of, significantly improves the performance of social learners through information-sharing. The core idea is to integrate a graph/network, where each node represents an agent, and the edges identity information-sharing or other relations among agents \cite{yang2020laplacian}.  

The state-of-the-art literature mainly concerns two variants of the MAB problem: (1) Stochastic bandits, where each arm yields a reward from an unknown, time-invariant distribution \cite{lai1985asymptotically, auer2002finite}; and (2) Adversarial bandits, where the reward distribution of each arm may change adversarially at each time step \cite{auer2002nonstochastic}. However, some application scenarios do not fit in these two models. Specifically, in some applications, the arms’ reward distributions vary much less frequently which has difficulties in suiting an adversarial bandit model \cite{cao2019nearly}. 

For example, in recommendation systems where each item represents an arm and users' clicks are rewards, the users' preferences towards items are unlikely to be time-invariant or change significantly at all time steps \cite{zhou2020near}. Other examples include investment options and dynamic pricing with feedback \cite{yu2009piecewise}. Consequently, we investigate an ``intermediate setting'' of a multi-agent system, namely the piecewise-stationary model. In such a model, the reward distribution of each arm is piecewise-constant and shifts at some unknown time steps called the \textit{change points}.
\subsection{Related Work}
The rising importance of networked systems, the development of online information-sharing networks, and emerging novel concepts such as the system of systems, motivate studying the multi-agent multi-armed bandit (MAMAB) problem as a framework to model distributed decision-making. And the MAMAB provides the required “learning-to-coordinate” framework in a distributed fashion \cite{hanawal2021multiplayer}. 

Reference \cite{buccapatnam2013multi} investigates the multi-armed bandit problem with side observations from the network, where an external agent decides for the network's users. The authors propose the $\epsilon$-greedy-LP strategy, which explores the action for each user at a rate that is a function of its network position. Reference \cite{kalathil2014decentralized} studies the distributed MAMAB problem and proposes an online indexing policy based on distributed bipartite matching. The method ensures that the expected regret grows no faster than $\mathcal{O}(\log^2T)$. However, the proposed policies in \cite{buccapatnam2013multi,kalathil2014decentralized} are agnostic of the network structure and do not consider the agents' heterogeneity. In \cite{madhushani2019heterogeneous}, the authors introduce the notion of sociability to model the likelihood probabilities that one agent observes its neighbors' choices and rewards in the network graph. The proposed algorithm has $\mathcal{O}(\log T)$ as regret bound. Reference \cite{madhushani2020dynamic} assumes that each agent observes the instantaneous rewards and choices of all its neighbors only when exploring. Based on such an assumption, it proposes an algorithm with a regret-bound $\mathcal{O}(\log T)$. In \cite{madhushani2019heterogeneous, madhushani2020dynamic,madhushani2021heterogeneous}, the observation between agents are opportunistic and occasional, whereas, in the real-world applications like social networks, the observations are continuous. In \cite{nayyar2016regret}, it is assumed that the communication among agents is only allowed by playing arms in a certain way. Two decentralized policies $E^3$ and $E^3-TS$ are proposed for both single player MAB problem and multi-player MAB problem respectively, where $E^3$ stands for Exponentially spaced Exploration and Exploitation policy and TS stands for Thompson sampling. In \cite{martinez2019decentralized}, the problem setting includes several agents that synchronously play the same MAB game. They develop a gossip-based algorithm that guarantees $\mathcal{O}(\log T)$ expected regret bound. Reference \cite{landgren2021distributed} considers the same problem with and without collisions. The authors propose two algorithms and investigate the influence of communication graph structure on group performance. In \cite{hanawal2021multiplayer}, the stochastic MAMAB problem is studied. Different as previous research, \cite{hanawal2021multiplayer} considers a dynamic scenario, where the players enter and leave at any time. Algorithms based on ``Trekking approach'' are proposed and sublinear regret is guaranteed with high probability in the dynamic scenario. Reference \cite{lalitha2021bayesian} investigates a decentralized MAMAB problem, where the arm's reward distribution is the same for every agent while the information exchange is limited to be at most $\text{poly}(K)$ times, $K$ is the number of arms. A decentralized Thompson Sampling algorithm and a decentralized Bayes-UCB algorithm are proposed to solve the formulated MAMAB problem.

However, most previous research on MAMAB focuses on either the stochastic- or adversarial environment, while it largely neglects the piecewise-stationary setting. References \cite{kocsis2006discounted,garivier2011upper} introduce the piecewise-stationary MAB, where the reward distributions of arms remain stationary for some intervals (piecewise-stationary segments) but change abruptly at some potentially unknown time steps (change points). The piecewise-stationary environment is specifically beneficial for modeling the real environments\cite{wang2019aware}, giving rise to several decision-making methods, especially for the single-agent MAB. The cutting-edge research on piecewise-stationary MAB includes two methods categories: passively adaptive and actively adaptive. The former makes decisions based on the most recent observations while unaware of the underlying distribution changes \cite{garivier2011upper,besbes2014stochastic,wei2018abruptly}. The latter incorporates a change point detector subroutine to monitor the reward distributions and restart the algorithm once a change point is detected \cite{besson2019generalized,auer2019adaptively,zhou2020near}. Below, we discuss some examples of the two categories in single-agent MAB emphasizing adaptive methods due to their higher relevance to our research.

Algorithm \textit{Discounted UCB} (D-UCB) \cite{kocsis2006discounted, yu2009piecewise,garivier2011upper} averages the past rewards with a discount factor, which results in $\mathcal{O}(\sqrt{NT}\log T)$ as regret bound. In \cite{garivier2011upper}, the authors propose the \textit{sliding-window UCB} (SW-UCB) method, which incorporates the past observations only within a fixed-length moving window for decision-making. The authors prove the regret bound $\mathcal{O}(\sqrt{NT}\log T)$. In \cite{liu2018change}, the authors propose a change-detection-based framework that actively detects the change points and restarts the MAB indices. They establish the regret bound $\mathcal{O}(\sqrt{NT \log \frac{T}{N}})$ for their proposed framework. The \textit{Monitored-UCB} (M-UCB) algorithm \cite{cao2019nearly} has a $\mathcal{O}(\sqrt{NMT\log T})$ regret bound, which is slightly higher than CUMSUM-UCB; nonetheless, it is more robust as it requires little parameter specification. Reference \cite{besson2019generalized} combines the bandit algorithm KL-UCB with a parameter-free change point detector, namely, the Bernoulli Generalized Likelihood Ratio Test (GLRT), to obtain a dynamic decision-making policy. The authors develop two variants, global restart, and local restart. They also prove the regret upper-bound $\mathcal{O}(\sqrt{NT\log T})$ for known number of change points $N$. Similarly, \cite{zhou2020near} applies GLRT change point detector in solving a combinatorial semi-bandit problem in a piecewise-stationary environment, where the regret is upper bounded by $\mathcal{O}(\sqrt{NMT\log T})$.
\subsection{Our Contribution}
Our main contributions are as follow:
\begin{itemize}
\item We propose an efficient running consensus algorithm for piecewise-stationary MAMAB, called \textit{RBO-Coop-UCB}. It addresses the multi-agent decision-making problem in changing environments by integrating a change point detector proposed in \cite{alami2020restarted}, namely, Restarted Bayesian online change point detector (RBOCPD) to a cooperative UCB algorithm. 
\item To improve the decision-making performance via information-sharing, we incorporate an efficient cooperation framework in the proposed strategy RBO-Coop-UCB: 1) The agents share observations in neighborhoods to enhance the performance in arm selection. 2) We integrate a majority voting mechanism into the restart decision part. The cooperation framework is generic and easily integrable to solve piecewise-stationary bandit problems, especially in various actively adaptive MAB policies.
\item For any networked multi-agent systems, we establish the group regret bound $\mathcal{O}(KNM\log T + K\sqrt{MT\log T})$. To the best of our knowledge, this is the first regret bound analysis for bandit policies with a RBOCPD.
\item By intensive experiments on both synthetic and real-world datasets, we show that our proposed policy, RBO-Coop-UCB, performs better than the state-of-the-art policies. We also integrate our proposed cooperative mechanism to different bandit policies and demonstrate performance improvement as a result of cooperation. 
\end{itemize}
The paper is structured as follows. We formulate the problem in \textbf{Section~\ref{sec:prob-form}}. In \textbf{Section~\ref{sec:rbo-coop-ucb}}, we first explain the restarted Bayesian online change point detector, which utilized in the proposed algorithm. Then we develop an algorithmic solution to the formulated problem. In \textbf{Section~\ref{sec:per-ana}}, we establish the theoretical guarantees, including the regret bound, for the proposed solution. In \textbf{Section~\ref{sec:expriment}}, we evaluate our proposal via numerical analysis based on both synthetic- and real-world datasets. \textbf{Section~\ref{sec:concl}} concludes the paper.
\section{Problem Formulation}
\label{sec:prob-form}
We consider an $M$-armed bandit with $K$ players (agents, hereafter) gathered in the set $\mathcal{K}$. The agents play the same piecewise-stationary MAB problems simultaneously for $T$ rounds. We use $\mathcal{M}$ to denote the time-invariant action set. At time $t$, the reward associated with arm $m \in \mathcal{M}$ is randomly sampled from distribution $f_t^m$ with mean $\mu_t^m$. The rewards are independent across the agents and over time. The agents form a network modeled by an undirected graph $\mathcal{G}(\mathcal{K}, \mathcal{E})$ where $\mathcal{E} = \{e(k,j)\}_{k, j \in \mathcal{K}}$ is the edge set. In $\mathcal{G}$, nodes and edges represent agents and the potential of communication, respectively. Two agents $k$ and $j$ are neighbors if $e(k,j) \in \mathcal{E}$. In addition to its own observation, each agent can observe its neighbors' selected arms and sampling rewards.

At each time step, each agent $k \in \mathcal{K}$ pulls one arm $m \in \mathcal{M}$ and obtains a reward sampled from $f_t^m$. We use $I_t^k$ to show the action of agent $k$ at time $t$. Besides, $X_t^m$ is the sampling reward of arm $m$ at time $t$. We assume that the reward distributions of arms are piecewise-stationary satifying the following Assumption~\ref{def:ps-Bernoulli}. 
\begin{assumption}[Piecewise-stationary Bernoulli process \cite{alami2020restarted}]
\label{def:ps-Bernoulli} 
The environment is piecewise-stationary, meaning that it remains constant over specific periods and changes from one to another. Let $T$ denote the time horizon and $N$ the overall number of piecewise-stationary segments observed until time $T$.
\begin{gather}
N = 1 + \sum_{t=1}^{T-1} \mathbbm{1}\{f_t^m \neq f_{t+1}^m \ \text{for some} \ m \in \mathcal{M}\}. \label{eq:num-ps}
\end{gather}
The reward distributions of arms are piecewise-stationary Bernoulli processes $\mathcal{B}(\mu_t^m)$ such that there exists an non-decreasing change points sequence $(\nu_n)_{n \in [1,N-1]} \in \mathbb{N}^{N-1}$ verifying
\begin{gather}
\begin{cases} 
\text{$\forall n \in [1,N-1], \quad \forall t \in  [\nu_n,\nu_{n+1}), \ \forall m \in \mathcal{M}$,}&\ \text{$\mu_t^m = \mu_n^m$}\\
\text{$\nu_1 = 1 < \nu_2 < \ldots < \nu_{N}=T$}&\ \text{}
\end{cases} \notag 
\end{gather}
\end{assumption}
The performance of each single agent $k$ is measured by its (dynamic) regret, the cumulative difference between the expected reward obtained by an oracle policy playing an optimal arm $I_t^*$ at time $t$, and the expected reward obtained by action $I_t^k$ selected by agent $k$ 
\begin{gather}
R_T^k = \sum_{t=1}^T [\mathbb{E}(X_t^{I_t^*}) - \mathbb{E}(X_t^{I_t^k})].
\end{gather}
Reference \cite{zinkevich2003online} introduces the concept of dynamic regret. In contrast to a fixed benchmark in the static regret, dynamic regret compares with a sequence of changing comparators and therefore is more suitable for measuring the performance of online algorithms in piecewise-stationary environments \cite{zhao2021bandit}. In the multi-agent setting considered here, we study the network performance in terms of the regret experienced by the entire network; As such, the define the decision-making objective as to minimize the expected cumulative \textit{group regret},
\begin{gather}
R_T=K\sum_{t=1}^T\mathbb{E}(X_t^{I_t^*})-\sum_{k=1}^K \sum_{t=1}^T \mathbb{E}(X_t^{I_t^k}). 
\label{eq:net-rgt}
\end{gather}
%
\section{The RBO-Coop-UCB Algorithm} 
\label{sec:rbo-coop-ucb}
\textbf{Notations.} We use boldface uppercase letters to represent vectors and use calligraphic letters to represent sets. For example, $\boldsymbol{X}_{s:t}$ denotes the sequence of observations from time step $s$ to $t$ and $\mathcal{X}^{k,m}$ refers to the set containing the sampled rewards of arm $m$ by agent $k$. The inner product is denoted by $\langle \cdot, \cdot \rangle$. $\mathcal{O}(\cdot)$ refers to the big-O notation while $o(\cdot)$ refers to the small-o notation. $\mathbbm{1}\{\cdot\}$ denotes the indicator function. \textbf{Table~\ref{tab:nota}} gathers the most important notations.

The proposed algorithm, RBO-Coop-UCB, combines a network UCB algorithm (similar to \cite{madhushani2021heterogeneous,landgren2021distributed}) with a change detector running on each arm (based on Bayesian change point detection strategy \cite{alami2020restarted}). Roughly-speaking, RBO-Coop-UCB involves three ideas:\\ 
(1) A unique cooperative UCB-based method that guides the network system towards the optimal arm in a piecewise-stationary environment;\\
(2) A change point detector introduced in \textbf{Algorithm~\ref{alg:rbocpd}};\\ 
(3) A novel cooperation mechanism for change point detection to filter out the false alarms.\\
We summarize the proposed policy in \textbf{Algorithm~\ref{alg:r-coop-ucb}}.

To the best of our knowledge, this is the first attempt to solve the MAMAB problem in a piecewise-stationary environment. Compare to the previous work \cite{landgren2021distributed,madhushani2019heterogeneous,madhushani2020dynamic,madhushani2021heterogeneous}, our proposed algorithm is applicable to different multi-agent systems and non-stationary environment. In the following, we first introduce the restarted Bayesian online change point detection procedure (RBOCPD), which we utilize to develop our decision-making policy, RBO-Coop-UCB, which we then explain in detail.
\begin{table}[!ht]
\begin{center}
\centering
\captionsetup{justification=centering}
\caption{Notation}
\label{tab:nota}
 \begin{tabular}{|c|p{6.7cm}|}
 \hline
 Notation & Meaning  \\ [0.5ex] 
 \hline
 $\boldsymbol{X}_{s:t}$ & $\boldsymbol{X}_{s:t} = (X_s, \ldots, X_t)$: Sequence of observations from time $s$ up to time $t \geq s$ \\
 \hline
 $\hat{\mu}_{s:t}$ & $\hat{\mu}_{s:t} = \frac{\sum_{i=s}^t X_i}{n_{s:t}}$; Empirical mean over $\boldsymbol{X}_{s:t}$\\
 \hline
 $n_{s:t}$ & $n_{s:t} = t-s+1$; Length of $\boldsymbol{X}_{s:t}$ \\ \hline 
 $r_t$ & Current run-length \\ \hline 
 $l_{s,t}$ & Forecaster loss \\
 \hline
 $\vartheta_{r,s,t}$ & Forecaster weight \\
 \hline
 $W$ & $W \in \mathbb{R}^{K\times K}$ Observation matrix, $W_{k,j} = 1$ if $k$ and $j$ are neighbors  \\
 \hline
 $K$ & Number of agents  \\
 \hline
 $M$ & Number of arms  \\
 \hline
 $N$ & Number of change points \\
 \hline
 $\nu_n$ & The real $n$-th change point \\
 \hline
 $\tau_n$ & The $n$-th detected change point \\
 \hline
 $I_t^k$ & Arm selected by agent $k$ at time $t$ \\
 \hline
 $N_T^{k,m}$ & Total number of times that agent $k$ observes from arm $m$ \\
 \hline
 $\hat{\mu}_T^{k,m}$ & Estimated reward of arm $m$ by from agent $k$ \\
 \hline
 $f_m^t$ & Reward distribution of arm $m$ at $t$ \\ 
 \hline
 $\mu_m^t$ & Expected reward of arm $m$ at $t$ \\
 \hline
 $X_t^m$ & Sampling reward of arm $m$ at time $t$ \\
 \hline
 $S_T^{k,m}$ & Total reward that agent $k$ observes from arm $m$ \\ \hline
 $\eta_k$ & Degrees of agent $k$ in $\mathcal{G}$ \\ \hline
 \end{tabular}
 \end{center}
\end{table}
\subsection{Restarted Bayesian Online Change Point Detector}
Detecting changes in the underlying distribution of a sequence of observations \cite{alami2020restarted} is a classic problem in statistics. The Bayesian online change point detection appeared first in \cite{adams2007bayesian,fearnhead2007line}. Since then, several authors have used it in different contexts  \cite{mellor2013thompson,alami2017memory,maghsudi2020non,8110693}. Nevertheless, the previous research seldom considers the theoretical analysis of its performance bounds, such as false alarm rate and detection delay. In \cite{alami2020restarted}, the authors develop a modification of Bayesian online change point detection and prove the non-asymptotic guarantees related to the false alarm rate and detection delay. Remark~\ref{rem:comp-glr-rbo} in Appendix~\ref{app:rbocpd} compares RBOCPD and other change point detectors in detail.

Let $r_t$ denote the number of time steps since the last change point, given the observed data $\boldsymbol{X}_{1:t}$, generated from the piecewise-stationary Bernoulli process described in \textbf{Definition~\ref{def:ps-Bernoulli}}. The Bayesian strategy computes the posterior distribution over the current run-length $r_t$, i.e.,  $p(r_t|\boldsymbol{X}_{1:t})$ \cite{adams2007bayesian}. The following message-passing algorithm recursively infers the run-length distribution \cite{alami2020restarted}
\begin{align*}
p(r_t|\boldsymbol{X}_{1:t}) & \propto \\
& \, \sum_{r_{t-1}} \underbrace{p(r_t|r_{t-1})}_\text{hazard} \underbrace{p(X_t|r_{t-1},\boldsymbol{X}_{1:t-1})}_\text{UPM} p(r_{t-1}|\boldsymbol{X}_{1:t-1}).
\end{align*}
A simple example of the hazard function $h$ is a constant $h = \frac{1}{\lambda} \in (0,1)$ \cite{alami2020restarted}.

The RBOCPD algorithm assumes that each possible value of the run-length $r_t$ corresponds to a specific forecaster. The loss $l_{s,t}$ of forecaster $s$ at time $t$ related to underlying predictive distribution (UPM) $p(X_t|r_{t-1},\boldsymbol{X}_{s:t-1})$ then follows as
\begin{align}
l_{s:t} &= - \log Lp(X_t|\boldsymbol{X}_{s:t-1}),  \\
 &=  - X_t \log Lp(1|\boldsymbol{X}_{s:t-1}) - (1-X_t)\log Lp(0|\boldsymbol{X}_{s:t-1}), \notag
\label{eq:fore-loss}
\end{align}
where $Lp(\cdot)$ is the Laplace predictor \cite{alami2020restarted}, defined below.
\begin{definition}[Laplace predictor] 
\label{def:lapp}
The Laplace predictor $Lp(X_{t+1}|\boldsymbol{X}_{s:t})$ takes as input a sequence $\boldsymbol{X}_{s:t} \in \{0,1\}^{n_{s:t}}$ and predicts the value of the next observation $X_{t+1} \in \{0,1\}$ as
\begin{gather}
Lp(X_{t+1}|\boldsymbol{X}_{s:t}) = \begin{cases} 
\text{$\frac{\sum_{i=s}^t X_i + 1}{n_{s:t}+2}$,}&\quad\text{if $X_{t+1} = 1$,}\\
\text{$\frac{\sum_{i=s}^t(1-X_i)+1}{n_{s:t}+2}$,}&\quad\text{if $X_{t+1} = 0$,}
\end{cases} 
\end{gather}
where $\forall X\in \{0,1\}, Lp(X|\phi) = \frac{1}{2}$ corresponds to the uniform prior given to the process generating $\mu_c$.
\end{definition}
The weight $\vartheta_{r,s,t}$ of forecaster $s$ at time $t$ for starting time $r$ is the posterior $\vartheta_{r,s,t} = p(r_t = t-s|\boldsymbol{X}_{s:t})$, where
\begin{gather}
\vartheta_{r,s,t} = \begin{cases} 
\text{$\frac{\eta_{r,s,t}}{\eta_{r,s,t-1}}\exp(-l_{s,t})\vartheta_{r,s,t-1}$,}&\quad\text{$\forall s<t$}\\
\text{$\eta_{r,t,t}\times \mathcal{V}_{r,t-1}$,}&\quad\text{$s=t$} 
\end{cases} 
\label{eq:theta}
\end{gather}
by using the hyperparameter $\eta_{r,s,t}$ (instead of the constant hazard function value $\frac{1}{\lambda}$) and the initial weight $\mathcal{V}_{r,t-1}$. The initial weight $\mathcal{V}_{r,t-1}$ is defined as
\begin{gather}
\mathcal{V}_{r:t-1} = \exp(-\hat{L}_{r:t-1}),
\end{gather}
for some starting time $r$, where $\hat{L}_{r:t-1} = \sum_{r' = r}^{t-1} l_{r':t-1}$ is the cumulative loss incurred by the forecaster $r$ from time $r$ until time $t-1$. Based on \eqref{eq:fore-loss}, the cumulative loss yields
\begin{gather}
\hat{L}_{r:t-1} = \sum_{r' = r}^{t-1} -\log Lp(x_{t-1}|\boldsymbol{x}_{r':t-2}).
\end{gather}
Besides, RBOCPD includes a restart procedure to detect changes based on the forecaster weight. For any starting time $r \leq t$,
\begin{gather}
\textbf{Restart}_{r:t} = \mathbbm{1}\{\exists s \in (r,t]: \vartheta_{r,s,t} > \vartheta_{r,r,t}\}. 
\label{eq:res}
\end{gather}
The intuition behind the criterion $\textbf{Restart}_{r:t}$ is the following: At each time $t < \nu$ with no change, the forecaster distribution concentrates around the forecaster launched at the starting time $r$. Thus, if the distribution $\vartheta_{r,s,t}$ undergoes a change, a change becomes observable. 

The restarted version of Bayesian Online Change Point Detector \cite{alami2020restarted} can be formulated as follows.
\begin{algorithm}[!htp] 
\label{alg:rbocpd}
\caption{R-BOCPD: RBO($X_{1:t}, \eta_{1,s,t}$)} 
\label{alg:rbocpd}
\textbf{Require:} Observations $X_{1:t}$ and hyperparameter $\eta_{1,s,t}$; 
\begin{algorithmic}[1]
\STATE $r \leftarrow 1, \vartheta_{r,1,1} \leftarrow 1, \eta_{r,1,1} \leftarrow 1$.
\FOR{$i=1,2,\dots, t$}
\STATE Calculate $\vartheta_{r,s,i}$ of each forecaster $s$ according to \eqref{eq:theta}.
\STATE Calculate $\textbf{Restart}_{r:i}$ according to \eqref{eq:res}.
\IF{$\textbf{Restart}_{r:i} = 1$}
\RETURN True
\ENDIF
\ENDFOR
\RETURN False
\end{algorithmic}
\end{algorithm}
\subsection{RBO-Coop-UCB Decision-Making Policy}
The proposed algorithm is a Network UCB algorithm that allows for some restarts.  Each agent runs the RBO-Coop-UCB in parallel to choose an arm to play. Also, each agent observes its neighbors' choices and the corresponding rewards. To guarantee enough samples for change point detection, each arm will be selected several times in the exploration steps. Let $I_t^k$ and $X_t^{k,m}$ be some variables to denote the selected arm and received reward of agent $k$ at time $t$, respectively, where $X_t^{k,m}$ is the i.i.d copy of $X_t^m$. The total number of times that agent $k$ observes option $m$'s rewards yields
\begin{gather}
N_T^{k,m} = \sum_{t=1}^T\sum_{j=1}^K \mathbbm{1}\{I_t^j = m\}\mathbbm{1}\{e(k,j) \in \mathcal{E}\}.
\end{gather}
The empirical rewards of arm $m$ by agent $k$ at time $T$ is
\begin{gather}
\hat{\mu}_T^{k,m} = \frac{S_T^{k,m}}{N_T^{k,m}},
\end{gather}
where $S_T^{k,m} = \sum_{t=1}^T\sum_{j=1}^K X_t^{k,m}\mathbbm{1}\{I_t^k = m\}\mathbbm{1}\{e(k,j) \in \mathcal{E}\}$ is the total reward observed by agent $k$ from option $m$ in $T$ trials. At every sampling time step, if the agent $k$ is in a forced exploration phase, it selects the arm using \eqref{eq:f-expl} to ensure a sufficient number of observations for each arm. Otherwise, it chooses the arm according to the sampling rule described in \textbf{Definition~\ref{def:sample-rule}}. 
\begin{definition}
\label{def:sample-rule}
The sampling rule $\{I_t^k\}_1^T$ for agent $k$ at time $t \in \{1,2,\ldots,T\}$ is
\begin{gather}
\mathbbm{1}\{I_t^k = m\} = \begin{cases} 
\text{$1$,}&\quad\text{if $Q_t^{k,m} = \max \{Q_t^{k,1}, \ldots, Q_t^{k,M}\}$}\\
\text{0,}&\quad\text{otherwise}
\end{cases} 
\end{gather}
with
\begin{align}
Q_t^{k,m} &= \hat{\mu}_t^{k,m} + C_t^{k,m}, \label{eq:q}\\
C_t^{k,m} &= \sqrt{\frac{\xi(\alpha^k+1)\log (t-\tau^k)}{N_t^{k,m}}},
\end{align}
where $\tau^k$ is the latest change points detected by agent $k$. $\xi \in (0,1]$ is a constant and $\alpha^k = \frac{\eta_k - \eta_k^{avg}}{\eta_k}$ is an agent-based parameter where $\eta_k$ is the number of neighbors of agent $k$. $\eta_k = \langle \boldsymbol{1}, W_k \rangle - 1$ with $\boldsymbol{1}$ a $K$ dimensional vector with all elements equal to $1$, $W_k$ is the $k$-th column of matrix $W$. $\eta_k^{avg} = \frac{1}{\eta_k}\sum_{e(k,j)\in \mathcal{E}}^K \eta_j$ is the average degree of neighbors of agent $k$. We assume that $\forall k \in \mathcal{K}$, $\eta_k^{avg} \leq 2\eta_k$, which indicates $\forall k \in \mathcal{K}$, $\alpha_k \in (-1,1)$.
\end{definition}
\begin{remark}
Different values of $\alpha^k$ imply heterogeneous agents' exploration. On the one hand, agents with more neighbors have more observations that reduce the uncertainties of their reward estimations and increase their exploitation potential. Less exploration then lowers the usefulness of the information they broadcast, thus decreasing its neighbors' exploitation potential \cite{madhushani2021heterogeneous}. Therefore, to improve the group performance, we propose the heterogeneous explore-exploit strategies with sampling rule in \textbf{Definition~\ref{def:sample-rule}} that regulate exploitation potential across the network.  
\end{remark}
After agent $k$ receives the sampling reward $X_t^{k,m}$ and observes its neighbors' sampling rewards, it combines the observed rewards into the observation collection $\mathcal{X}^{k,m}$ to run the RBOCPD (\textbf{Algorithm~\ref{alg:rbocpd}}). In general, the set $\mathcal{X}^{k,m} = \boldsymbol{X}_{\tau^k:N_t^{k,m}}$ contains all the sampling rewards of arm $m$ observed by agent $k$ since the last change point $\tau^k$. Each agent $k$ receives a binary restart signal $r_t^{k,m}$ afterward, where $r_t^{k,m} = 1$ if there is a change point and zero otherwise. The agent calculates the restart signal $r_t^{k,m}$ using \eqref{eq:res}. The agents make final restart decisions using the cooperation mechanism described in \textbf{Definition~\ref{def:coop-res}} based on its restart signal and observations.
\begin{definition}[Cooperative restart mechanism]
\label{def:coop-res}
Each agent makes the final restart decision using the majority voting outcome among neighbors: If more than half of its neighbors detect a change in one of the played arms, then the agent restarts its UCB indices as
\begin{gather}
\sum_{j \in \mathcal{N}_k} \mathbbm{1}\{r_t^{j,m} > 0\} \geq \lceil \frac{\eta_k}{2} \rceil \rightarrow \textbf{Restart}_t^k = \textbf{True}.
\end{gather}
However, different agents have distinct observations, so a simple majority voting mechanism at each time step might lead to miss-detection due to the agent's asynchronous detection. Therefore, we propose an efficient cooperation mechanism for restart decision, where a restart memory time window records the previous restart of agents' neighbors in a short period $d$. The cooperation restart detection considers the majority voting of restart among neighbors in that period, i.e.,
\begin{align}
    \sum_{j \in \mathcal{N}_k} \mathbbm{1}\{\exists i \in [N_{t-d}^{j,m}, N_t^{j,m}], r_i^{j,m} &> 0\} \geq \lceil \frac{\eta_k}{2} \rceil \notag \\
    &\rightarrow \textbf{Restart}_t^k = \textbf{True}.
\end{align}
Hence, the slower detector receives the restart information from the faster ones to prevent missing change points. 
\end{definition}
\begin{remark}
The length of restart memory time window $d$ depends on the detection delay $d = \mathcal{D}_{\Delta, \nu_{n-1}+d_{n-1}^{k,m},\nu_n}$ of RBOCPD, as shown in \textbf{Theorem~\ref{the:dd}}. Although each agent maintains a change point detector with its observations, which could be faster or slower, all detectors follow the same principle. We bound the detection delay in \textbf{Theorem~\ref{the:dd}}; Therefore, for every change point, the maximum detection time difference is bounded. As a result, selecting $d$ based on the delay to include all possible correct detections improves the regret bound.
\end{remark}
\begin{algorithm}[!htp]
\caption{RBO-Coop-UCB} 
\label{alg:r-coop-ucb}
\begin{algorithmic}[1]
\STATE \textbf{Initialization} $\forall m \in \mathcal{M}$, $\forall k \in \mathcal{K}$, $\mathcal{X}^{k,m} \leftarrow \phi $; $N_0^{k,m} \leftarrow 0$; $S_0^{k,m} \leftarrow 0$; $\tau^k \leftarrow 0 $.
\FOR{$t = 1,2, \cdots, T$}
\FOR{$K \in \mathcal{K}$ in parallel }
\IF{$(t - \tau^k) \mod \lfloor \frac{M}{p} \rfloor \in \mathcal{M}$}
\STATE Select arm $I_t^k$ \hfill (forced exploration)
\begin{gather}
I_t^k \leftarrow (t - \tau^k) \mod \lfloor \frac{M}{p} \rfloor   \label{eq:f-expl}
\end{gather}
\ELSE 
\STATE Select arm $I_t^k$ as \hfill (UCB)
\begin{gather}
I_t^k \leftarrow \text{argmax}_m (\hat{\mu}_t^{k,m} + C_t^{k,m}),
\end{gather}
where $\hat{\mu}_t^{k,m} = \frac{S_t^{k,m}}{N_t^{k,m}}$ and $C_t^{k,m} = \sqrt{\frac{\xi(\alpha^k+1)\log (t-\tau^k)}{N_t^{k,m}}}$. 
\ENDIF
\STATE Play arm $I_t^k$ and receive the reward $X_t^{k,I_t^k}$.
\STATE Observe neighbors' option and rewards and update 
\begin{gather*}
N_t^{k,m} \leftarrow N_{t-1}^{k,m}+\sum_{j=1}^K \mathbbm{1}\{I_t^j = m\}\mathbbm{1}\{e(k,j)\in\mathcal{E}\}, \\
S_t^{k,m} \leftarrow \sum_{\tau=1}^t\sum_{j=1}^K X_t^{k,m}\mathbbm{1}\{I_t^j = m\}\mathbbm{1}\{e(k,j)\in\mathcal{E}\}, \\
\mathcal{X}^{k,m} \leftarrow  \mathcal{X}^{k,m} \cup \{X_t^{j,m}\}, \text{if} \ j \in \mathcal{N}_k, \mathbbm{1}\{I_t^j = m\} = 1
\end{gather*}
\STATE $r_t^{k,m} = \text{RBO}_k(\mathcal{X}^{k,m},\eta_{\tau^k,s,N_t^{k,m}})$ (\textbf{Algorithm~\ref{alg:rbocpd}}).
\IF{$\sum_{j \in \mathcal{N}_k} \mathbbm{1}\{\exists i \in [N_{t-d}^{j,m}, N_t^{j,m}], r_i^{j,m} > 0\} \geq \lceil \frac{\eta_k}{2} \rceil$}
\STATE $\tau^k \leftarrow t$, $\forall m \in \mathcal{M}$, $\mathcal{X}^{k,m} \leftarrow \phi $; $N_t^{k,m} \leftarrow 0$; $S_t^{k,m}\leftarrow 0$ \hfill (restart agent $k$'s UCB)
\ENDIF
\ENDFOR
\ENDFOR
\end{algorithmic}
\end{algorithm}
\section{Performance Analysis} 
\label{sec:per-ana}
In this section, we analyse the $T$-step regret of our proposed algorithm RBO-Coop-UCB.
\begin{assumption}
\label{assump:d}
Define $d_n^{k,m} = \lceil \frac{M}{p}\mathcal{D}_{\Delta,(\nu_{n-1}+d_{n-1}^{k,m}),\nu_n}+ \frac{M}{p}\rceil$, where $\mathcal{D}_{\Delta,r,\nu_n} = \min \{d \in \mathbb{N}^*: d > \frac{n_{r:\nu_n-1}(f_{r,\nu_n,d+\nu_n-1} - \log \eta_{r,\nu_n,d+\nu_n-1}) }{2n_{r:\nu_n-1}(\Delta-\mathcal{C}_{r,\nu_n,d+\nu_n-1,\delta})^2+ \log \eta_{r,\nu_n,d+\nu_n-1} - f_{r,\nu_n,d+\nu_n-1}}\}$ and $\mathcal{C}_{r,\nu_n,d+\nu_n-1,\delta} = \frac{\sqrt{2}}{2} (\sqrt{\frac{n_{r:\nu_n-1}+1}{n_{r:\nu_n-1}^2}\log(\frac{2\sqrt{n_{r:\nu_n}}}{\delta})} + \sqrt{\frac{n_{\nu_n:\nu_n+d-1}+1}{n_{\nu_n:\nu_n+d-1}^2}\log(\frac{2n_{r:d+\nu_n-1}\sqrt{n_{\nu_n:\nu_n+d-1}}\log^2(n_{r:d+\nu_n-1})}{\log 2 \delta})})$. Then we assume that for all $n \in \{1,\ldots,N\}$, $k \in \mathcal{K}$, $m \in \mathcal{M}$, $\nu_n - \nu_{n-1} \geq 2\max (d_{n}^{k,m},d_{n-1}^{k,m})$.
\end{assumption}
\begin{remark}
Assumption~\ref{assump:d} is a standard assumption in non-stationary multi-armed bandit literature \cite{cao2019nearly,besson2019generalized,zhou2020near}. It guarantees that the length between two change points is sufficient to detect the distribution change with high probability. Besides, the detection delay $\mathcal{D}_{\Delta,r,\nu_n}$ is asymptotically order optimal. The asymptotic regime is reached when $\frac{n_{r:\nu_n-1}}{\log (1/\delta)} \rightarrow \infty$, and $\log n_{r:\nu_n-1} = o(\log \frac{1}{\delta})$, when $\delta \rightarrow 0$, we obtain that \cite{alami2020restarted}
\begin{gather}
\mathcal{D}_{\Delta,r,\nu_n} \xrightarrow[\text{$\nu_n \rightarrow \infty$}]{} \frac{-\log \eta_{r,\nu_n, d+\nu_n-1}+ o(\log \frac{1}{\delta})}{2 |\mu_n - \mu_{n-1}|^2}.  \label{eq:asy-dd}
\end{gather}
If we choose the hyperparameter $\eta_{r,s,t} \approx \frac{1}{n_{r:t}}$ in RBOCPD and plug it into the asymptotic expression of detection delay \eqref{eq:asy-dd}, we get
\begin{gather}
    \mathcal{D}_{|\mu_n - \mu_{n-1}|,r,\nu_n} \xrightarrow[\text{$\nu_n \rightarrow \infty$}]{} \frac{o(\log \frac{1}{\delta})}{2 |\mu_n - \mu_{n-1}|^2}.
\end{gather}
\end{remark}
%
\begin{lemma}[Regret bound of stationary RBO-Coop-UCB] 
\label{lem:rgt-sta}
Consider a stationary environment, i.e., when $N = 1$, $\nu_0 = 0$, and $\nu_1 = T$. Then the upper bound of the expected cumulative regret of RBOCPD-Coop-UCB for each agent follows as
\begin{gather}
    \mathcal{R}_T^k \leq \Delta_1^*[M\sigma T + pT + M\lceil \frac{8 \xi \log T}{(\Delta_1^{\min})^2} \rceil + M(1+\frac{\pi^2}{3})], \notag 
\end{gather}
where $\Delta_1^*$ is the maximum gap between the mean rewards of arms, $\Delta_1^{\min}$ is the minimum gap between the mean rewards of arms, and $\sigma < \delta$ is the false alarm rate under cooperation.
\end{lemma}
\begin{IEEEproof}
See Appendix~\ref{app:rgt-sta}.
\end{IEEEproof}
\begin{remark}
\textbf{Lemma~\ref{lem:rgt-sta}} shows that the regret of RBO-Coop-UCB incorporates several sources: false alarm, forced exploration, and the classic regret of the UCB algorithm. 
\end{remark}
\begin{lemma}[False alarm probability in a stationary environment]
\label{lem:fap-sc}
Consider the stationary scenario, i.e., $N =1$, with confidence level $\delta > 0$. Define $\tau_1^{k,m}$ as the time of detecting the first change point of the $m$-th base arm. Let $\tau_1^k = \min_{m\in \mathcal{M}} \tau_1^{k,m}$, because RBO-Coop-UCB restarts the entire algorithm if a change point is detected on any of the base arms. The false alarm rate under cooperation is \cite{alami2020restarted,zhou2020near}
\begin{gather}
P(\tau_1^k \leq T) \leq \sum_{m=1}^M P(\tau_1^{k,m} \leq \tau) \leq  M\sigma
\end{gather}
\end{lemma}
\begin{IEEEproof}
See Appendix~\ref{app:fap-sc}.
\end{IEEEproof}
\begin{lemma} 
\label{lem:p-delay}
Define $\mathcal{C}_n^k$ as the event when all the change points up to the $n$-th one have been detected successfully by agent $k$ within a small delay. Formally,  \cite{besson2019generalized}:
\begin{gather}
\mathcal{C}_n^k = \{\forall i \leq n, \tau_i^k \in \{\nu_i+1, \ldots, \nu_i+d_i^k\} \}, 
\label{eq:gec}
\end{gather}
and 
\begin{gather}
P(\mathcal{C}_n^k) \leq P(\tau_n^k < \nu_n | \mathcal{C}_{n-1}^k) + P(\tau_n^k > \nu_n + d_n^k | \mathcal{C}_{n-1}^k) 
\end{gather}
Then, $(a) = P(\tau_n^k < \nu_n | \mathcal{C}_{n-1}^k) \leq M\sigma$ and $(b) = P(\tau_n^k > \nu_n + d_n^k | \mathcal{C}_{n-1}^k) \leq \delta$, where $\tau_n^k$ is the detection time of the $n$-th change point.
\end{lemma}
\begin{IEEEproof}
See Appendix~\ref{app:p-delay}.
\end{IEEEproof}
\begin{theorem}
\label{the:rgt}
Running \textbf{Algorithm~\ref{alg:r-coop-ucb}} with Assumption~\ref{assump:d}. Define the suboptimality gap in the $i$-th stationary segment as 
\begin{align}
    &\Delta_i^* =  \mu_t^* - \min_{m \in \mathcal{M}}\mu_t^m, \quad t \in [\nu_{i-1},\nu_i],   \notag \\
    &\Delta_i^{\min} = \mu_t^* - \max_{m \in \mathcal{M} / m*} \mu_t^m, \quad t \in [\nu_{i-1},\nu_i]. \notag 
\end{align}
then the expected cumulative regret of RBO-Coop-UCB with exploration probability $p$ and confidence level satisfies 
\begin{gather}
R_T^k \leq \sum_{i=1}^{N}\tilde{C}_i^k + \Delta^*T(p+2MN\sigma + M\delta),
\end{gather}
where $\tilde{C}_i^k = \Delta_i^*[M\lceil \frac{8 \log T}{(\Delta_i^{\min})^2} \rceil + M(1+\frac{\pi^2}{3})]$, and $\sigma < \delta$ is the maximum false alarm under cooperation.
\end{theorem}
\begin{IEEEproof}
See Appendix~\ref{app:the-rgt}.
\end{IEEEproof}
\begin{corollary}
\label{cor:rgt}
Let $\delta = \frac{1}{T}$ and $p = \sqrt{\frac{M\log T}{T}}$. Then the regret of RBO-Coop-UCB have the following upper bound:
\begin{gather}
R_T \leq \mathcal{O}(KNM\log T + K\sqrt{MT\log T})
\end{gather}
\end{corollary}
\begin{IEEEproof}
See Appendix~\ref{app:cor}.
\end{IEEEproof}
\section{Experiments}
\label{sec:expriment}
In this section, we evaluate the performance of our proposed RBO-Coop-UCB algorithm in different non-stationary environments using synthetic datasets and real-world datasets. We compare RBO-Coop-UCB with five baselines from the cutting-edge literature and one variant of RBO-Coop-UCB. More precisely, we use \textbf{DUCB} \cite{garivier2011upper} and \textbf{SW-UCB} (Sliding Window UCB) \cite{garivier2011upper} as passively adaptive benchmarks for the piecewise-stationary multi-armed bandit, and \textbf{M-UCB} (Monitored-UCB) \cite{cao2019nearly} and \textbf{GLR-UCB} \cite{maillard2019sequential,besson2019generalized} as actively adaptive ones. For consistency, in addition to the non-cooperative setting (each agent runs the algorithm independently), we implement each of the above piecewise-stationary algorithms also in a cooperative setting (information-sharing). In addition, we compare with \textbf{UCB} \cite{auer2002finite} as a stochastic bandit policy and \textbf{EXP3} \cite{auer2002nonstochastic} as an adversarial one. Finally, to validate the effectiveness of cooperation in change point detection, we implement GLR-Coop-UCB, which has a similar flow to RBO-Coop-UCB: It includes information-sharing and cooperative change point detection decision-making. The only difference between RBO-Coop-UCB and GLR-Coop-UCB is the implemented change point detector. 
The hyperparameter in our experiments are as follow.
\begin{itemize}
    \item UCB: None.
    \item DUCB: Discount factor $\gamma = 1 - (4B)^{-1}\sqrt{N/T}$, here $\gamma = 1- \sqrt{N/T}/4$.
    \item SW-UCB: Sliding window length $\tau = 2B\sqrt{T \log T/N}$, here $\tau = 2\sqrt{T \log T/N}$.
    \item M-UCB: $\delta = \max_{i \in N, m \in \mathcal{M}}|\mu_i^m - \mu_{i+1}^m|$, window size $\omega = 800$, $b = [\omega \log(2MT^2)/2]^{1/2}$, and $\gamma = 0.05 \sqrt{\frac{(N-1)(2b+3\sqrt{\omega})}{2T}}$.
    \item GLR-UCB and GLR-Coop-UCB: $\delta = \frac{10}{T}$, and $p =  \sqrt{\frac{\log T}{T}}$.
    \item RBO-UCB and RBO-Coop-UCB: $\eta_{r,s,t} = \frac{10}{T}$, and $p = \sqrt{\frac{\log T}{T}}$.
\end{itemize}
Table~\ref{tab:cpd} summarizes the performance of different change point detectors in all datasets. In the following section, we analyze the algorithms' performance in details. All the experimental results are based on ten independent runs.
\begin{table}[!htp]
    \centering
    \begin{tabular}{ |c|p{1.3cm}|p{1.3cm}|p{0.8cm}|p{1cm}| } 
    \hline
     Experiment &  Algorithm & Correct \newline Detection & Delay (step) & False Alarm   \\
    \hline
    \multirow{4}{1.5cm}{Synthetic Dataset } & \footnotesize{RBO-Coop} & 93/120 & 63.1 & 0.002\%\\ 
    \cline{2-5}
    & RBO & 108/120 & 106.7 & 0.024\%\\ 
    \cline{2-5} 
    & \footnotesize{GLR-Coop} & 60/120 & 98.8 & 0\\ 
    \cline{2-5} 
    & GLR & 60/120 & 99.3 & 0.0003\%\\
    \hline
    \multirow{4}{1.5cm}{Yahoo! Dataset } & \footnotesize{RBO-Coop} & 314/400 &234.8 & 0.001\%\\ 
    \cline{2-5}
    & RBO & 358/400 & 257.2 & 0.004\%\\ 
    \cline{2-5} 
    & \footnotesize{GLR-Coop}  & 92/400 & 787.2 & 0.005\%\\ 
    \cline{2-5} 
    & GLR & 77/400 & 730.3 & 0.006\%\\
    \hline
    \multirow{4}{1.5cm}{Digital Market Dataset } & \footnotesize{RBO-Coop} & 481/770 & 79.0 & 0.003\% \\ 
    \cline{2-5}
    & RBO & 556/770 & 105.1 & 0.010\% \\ 
    \cline{2-5} 
    & \footnotesize{GLR-Coop}  & 178/770 & 242.2 & 0.009\% \\ 
    \cline{2-5} 
    & GLR & 144/770 & 194.2 & 0.011\% \\
    \hline
    \end{tabular}
    \caption{Performance of Change Point Detectors.}
    \label{tab:cpd}
\end{table}
\subsection{Synthetic Dataset}\label{sec:e1}
In this section, we consider the following MAMAB setting with synthetic dataset:
\begin{itemize}
    \item The network consists of $K = 3$ agents and the agents face an identical MAB problem with $M = 5$ arms. 
    \item There are $N = 4$ piecewise-stationary Bernoulli segments, where only one base arm changes its distribution between two consecutive piecewise-stationary segments. 
\end{itemize}
\textbf{Figure~\ref{fig:e1-network}} and \textbf{Figure~\ref{fig:e1-env}} respectively show the network and the arms' reward distributions.
\begin{figure}[!ht]
\begin{subfigure}{.23\textwidth}
  \centering
  \includegraphics[width=.85\linewidth]{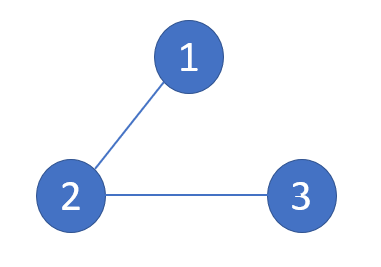}
  \caption{Observation network.}\label{fig:e1-network}
\end{subfigure}
\begin{subfigure}{.23\textwidth}
  \centering
  \vspace{-10pt}
  \includegraphics[width=.95\linewidth]{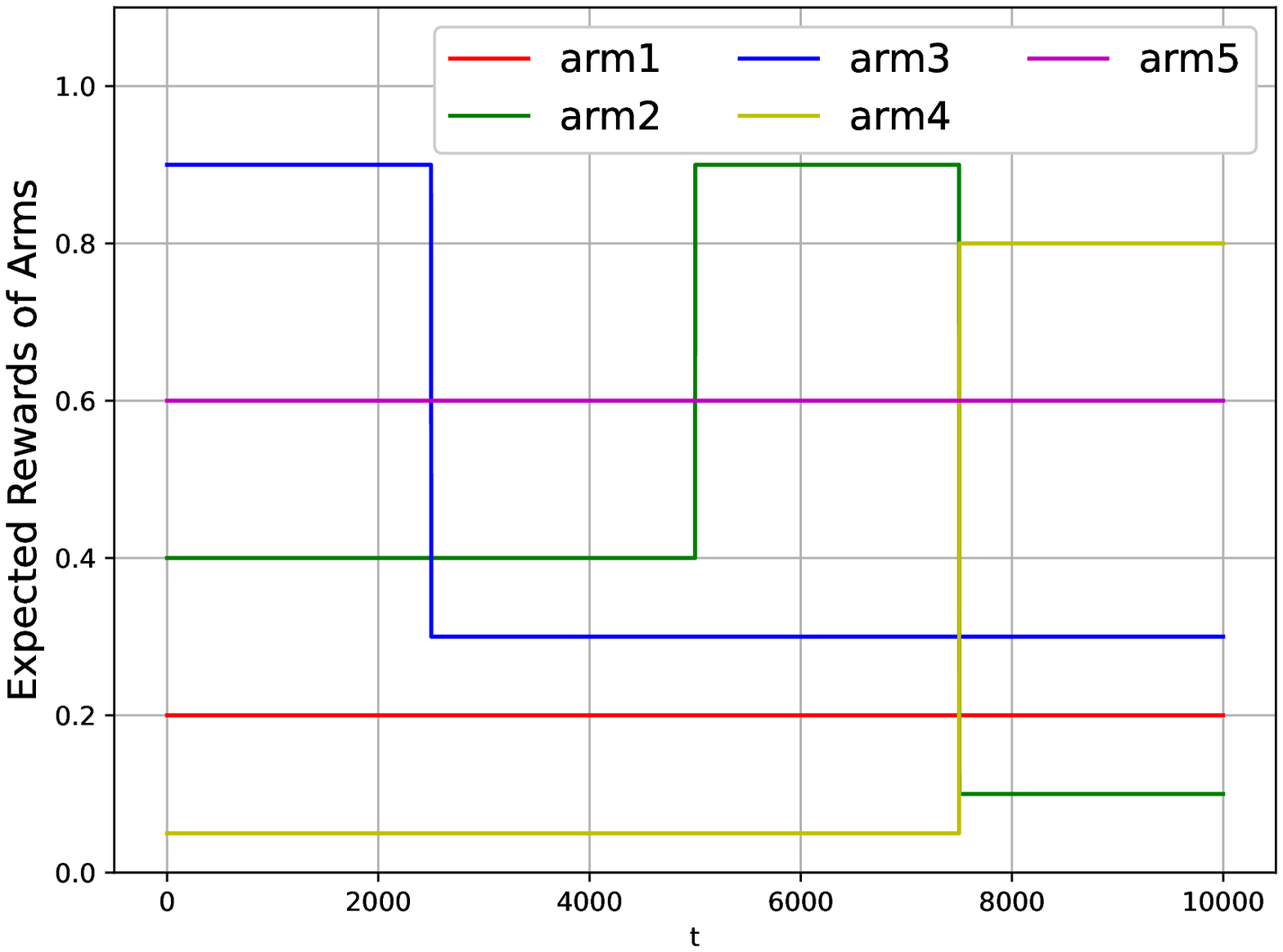}  
  \vspace{-5pt}
  \caption{Expected rewards of arms.}\label{fig:e1-env}
\end{subfigure}
\caption{Setting of Experiment I (synthetic data).}\label{fig:e1}
\end{figure}

\textbf{Figure~\ref{fig:e1-rgt}} summarizes the average regret of all algorithms.
\begin{figure}[!htp]
\vspace{-15pt}
  \centering
  \includegraphics[width=0.75\linewidth]{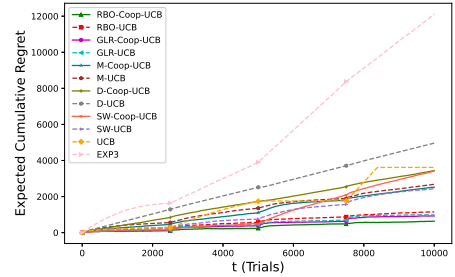}
  \caption{Expected cumulative regret for several benchmarks using a synthetic dataset.}
  \label{fig:e1-rgt}
  \vspace{-15pt}
\end{figure} 
Based on \textbf{Figure~\ref{fig:e1-rgt}}, RBO-Coop-UCB has the lowest regret among all piecewise-stationary algorithms. Compared with the algorithms with cooperation (such as RBO-Coop-UCB, GLR-Coop-UCB, and M-Coop-UCB, shown by solid lines), most algorithms without cooperation (such as RBO-UCB, GLR-UCB, and M-UCB, depicted by dashed lines) suffer from higher regrets. An exception is SW-UCB, whose regret is lower than that of SW-Coop-UCB. Our proposed cooperative framework is originally designed for actively adaptive algorithms (i.e., RBO-Coop-UCB, GLR-Coop-UCB, and M-Coop-UCB) and the cooperative restart decision-making does not exist in the passively adaptive algorithms (D-Coop-UCB, SW-Coop-UCB), that could be a reason for the higher regret in SW-Coop-UCB. In general, most algorithms have a lower regret than UCB and EXP3, which means that for optimal decision-making, it is essential to take into account the environment dynamics. Note that RBO-Coop-UCB performs better than GLR-Coop-UCB, whereas RBO-UCB has a higher regret than GLR-UCB. We investigate this effect in the following in detail. 

\textbf{Figure~\ref{fig:e1-cpd}} shows the histogram of time steps that different algorithms have identified as a change point. In \textbf{Figure~\ref{fig:e1-rbolog}}, there are three peaks ($t = 2500, 5000, 7000$). Especially, for $t=2500$ and $t = 7500$, the times to be regarded as change point reaches almost 30 (three agents ten times). That means, with a high probability, RBO-Coop-UCB, and RBO-UCB detect the changes. Compared with RBO-Coop-UCB, RBO-UCB has more false alarms. \textbf{Figure~\ref{fig:e1-glrlog}} demonstrates the change point detection of GLR-Coop-UCB and GLR-UCB. Compared with the algorithms with RBOCPD, those with GLRCPD benefit from fewer false alarms; however, they ignore the second change point $t = 5000$. \textbf{Figure~\ref{fig:e1-coop}} shows that RBO-Coop-UCB has a better performance than GLR-Coop-UCB in change point detection: It provides a higher correct detection rate for the first and third change points and has a smaller time step delay. Besides, RBO-Coop detects the GLR ``undetectable" change point (second change point), which indicates its wider detectable range. In \textbf{Figure~\ref{fig:e1-nocoop}}, RBO has a higher correct detection rate than GLR; nevertheless, the false alarm rate of RBO is also very high, which is not negligible. The statistics shown in Table~\ref{tab:cpd} also lead to the same conclusion. Frequent false alarm increases regret in RBO-UCB. In general, RBO-Coop and GLR-Coop have better detection than RBO and GLR, especially RBO-Coop, which proves the designed cooperation mechanism is effective and particularly suitable for the UCB algorithm incorporating RBOCPD. These are consistent with the conclusion drawn based on \textbf{Figure~\ref{fig:e1-rgt}}.
\begin{figure}[!ht]
\begin{subfigure}{.23\textwidth}
  \centering
  \includegraphics[width=.9\linewidth]{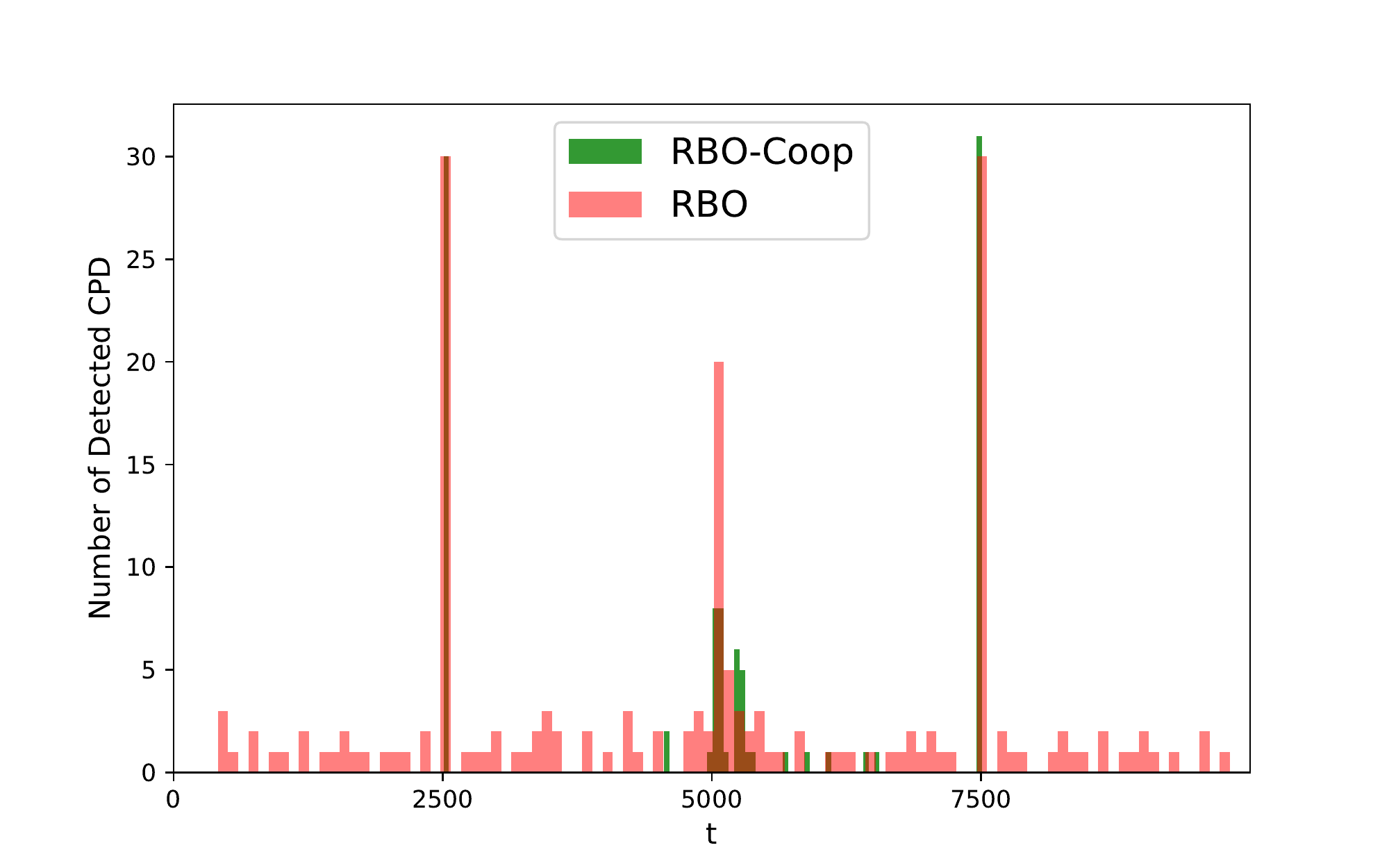}
 \caption{Change points detected by RBO-Coop and RBO.}
 \label{fig:e1-rbolog}
\end{subfigure}
\begin{subfigure}{.23\textwidth}
  \centering
  \includegraphics[width=.9\linewidth]{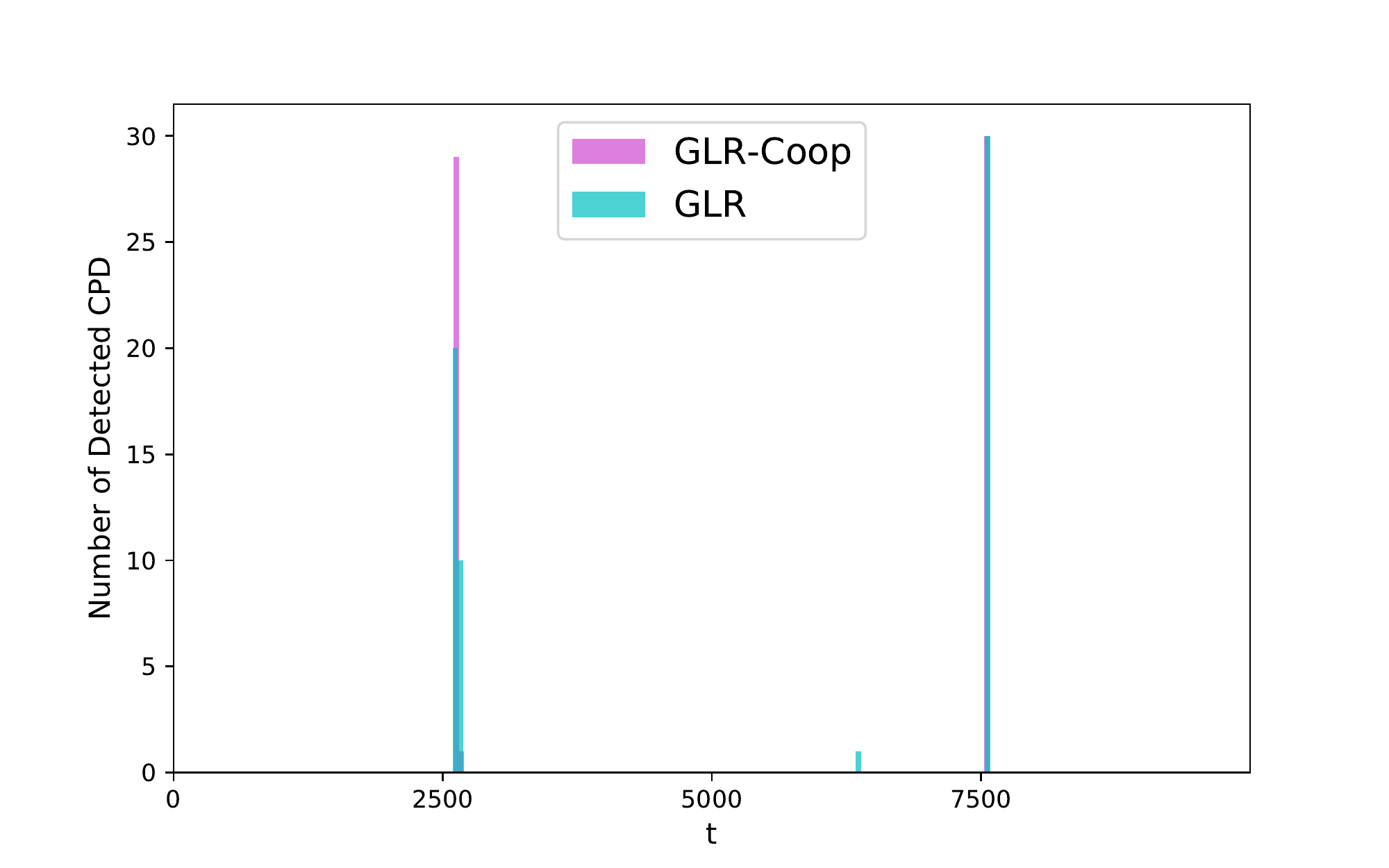}  
  \caption{Change points detected by GLR-Coop and GLR.}
  \label{fig:e1-glrlog}
\end{subfigure}
\newline
\begin{subfigure}{.23\textwidth}
  \centering
  \includegraphics[width=.9\linewidth]{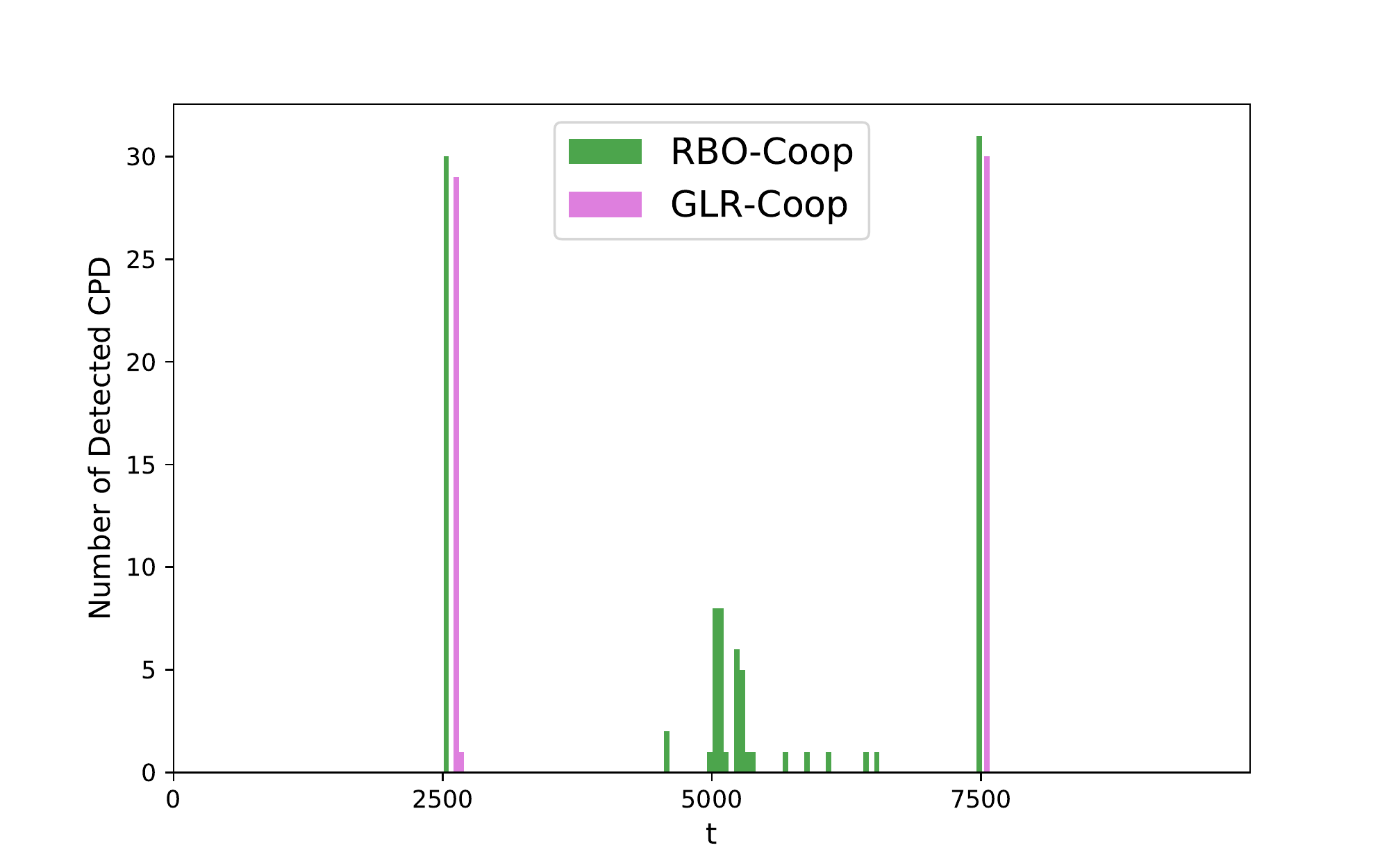}
  \caption{Change points detected by RBO-Coop and GLR-Coop.}
  \label{fig:e1-coop}
\end{subfigure}
\begin{subfigure}{.23\textwidth}
  \centering
  \includegraphics[width=.9\linewidth]{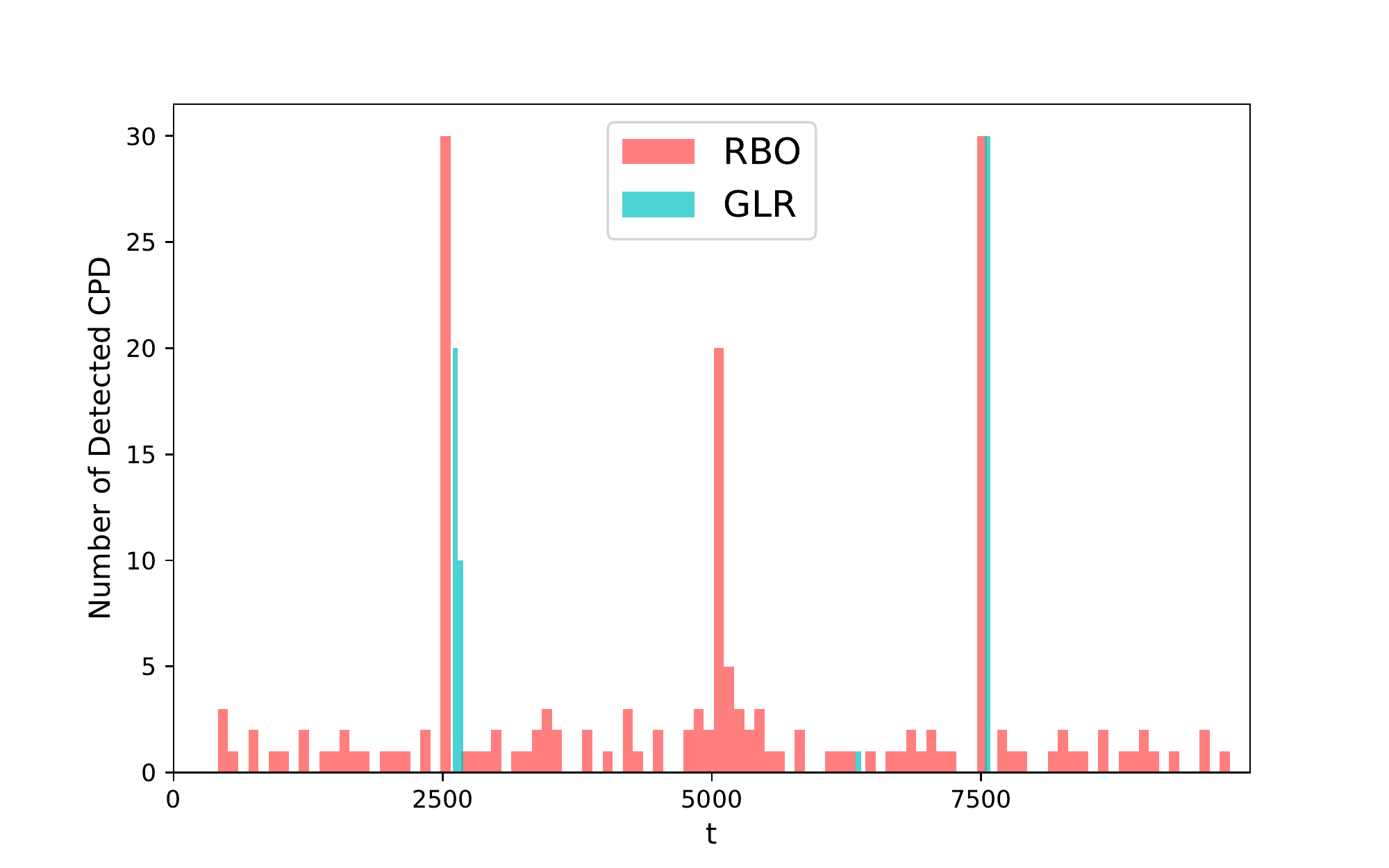}  
    \caption{Change points detected by RBO and GLR.}
    \label{fig:e1-nocoop}
\end{subfigure}
\caption{Change point detection in different algorithms on a  synthetic dataset.}
\label{fig:e1-cpd}
\end{figure}

%
To clarify the cooperation mechanism in change point detection, we show the detection record and the final decisions in Figure~\ref{fig:e1-details}. The figure illustrates the detection of each individual agent together with the final decision based on the neighborhood majority voting for arm 2. In that arm, the first change point occurs at $t=5000$. First, agent 2 detects that change, and then agent 1. Given the majority voting, agent 1 and 2 will restart their MAB policies after agent 1 detects the change. After several time steps, when agent 3 detects the change point, it has the restart history from its neighbor agent 2. Thus, it also restarts immediately at this time step. The same decision process follows by the second change at $t = 7500$. Agent 2 suffers from several false alarms between the two points, which, however, the cooperation mechanism filters out. Thus, in the final decision, this period does not include any restart. In conclusion, the developed cooperation mechanism significantly reduces the false alarm rate and improves the chances of correct change detection. 
\begin{figure}[!htp]
  \centering
  \vspace{-5pt}
  \includegraphics[width=0.6\linewidth]{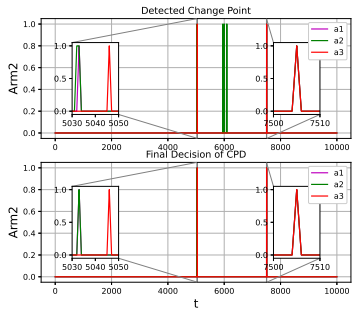}
  \caption{Detected Change Point and Final Decision.}
  \label{fig:e1-details}
  \vspace{-15pt}
\end{figure} 
\subsection{Real World Dataset}
In this section, we evaluate our proposed algorithm based on two different real-world datasets: Yahoo dataset\footnote{Yahoo! Front Page Today Module User Click Log Dataset on \url{https://webscope.sandbox.yahoo.com}} and digital marketing, which are same as the real-world dataset in \cite{cao2019nearly}. To be more detailed, in Yahoo dataset we verify the effectiveness in piecewise-stationary Bernoulli distribution, which is also our original design while with the digital marketing dataset we verify a more general situation. 
\subsubsection{Yahoo! Dataset}
We evaluate our proposed algorithm based on a real-world dataset from Yahoo, which contains the user click log for news articles displayed on the Yahoo! Front Page \cite{li2011unbiased}. To prepare the dataset, we follow the lines of \cite{cao2019nearly}. Therein, the authors preprocess the \textsl{Yahoo!} dataset by manipulating it. They then demonstrate the arm distribution information in Figure 3 of that paper.  To adapt to our multi-agent setting, we implement the experiment with $K=5$ agents. We enlarge the click rate of each arm by ten times \cite{zhou2020near}.\footnote{This step is essential because in the original setting, the gap between RBOCPD and GLRCPD can be smaller than the minimum detectable value.}~ \textbf{Figure~\ref{fig:e2-network}} and \textbf{Figure~\ref{fig:e2-env}} respectively show 
the network and the arms' reward distributions.
\begin{figure}[!ht]
\begin{subfigure}{.23\textwidth}
  \centering
  \includegraphics[width=.65\linewidth]{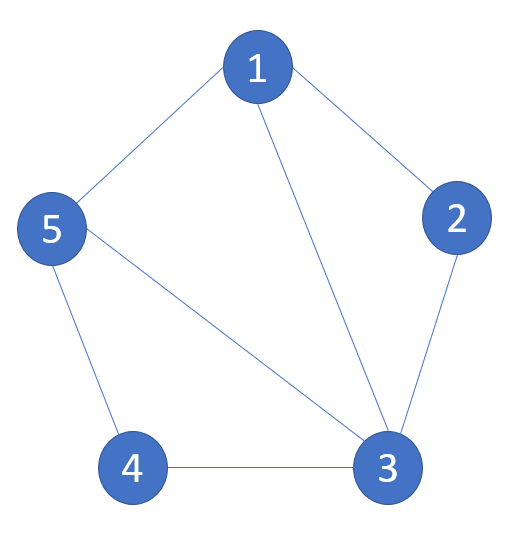}
  \caption{Observation network.}
  \label{fig:e2-network}
\end{subfigure}
\begin{subfigure}{.24\textwidth}
  \centering
  \vspace{-5pt}
  \includegraphics[width=.95\linewidth]{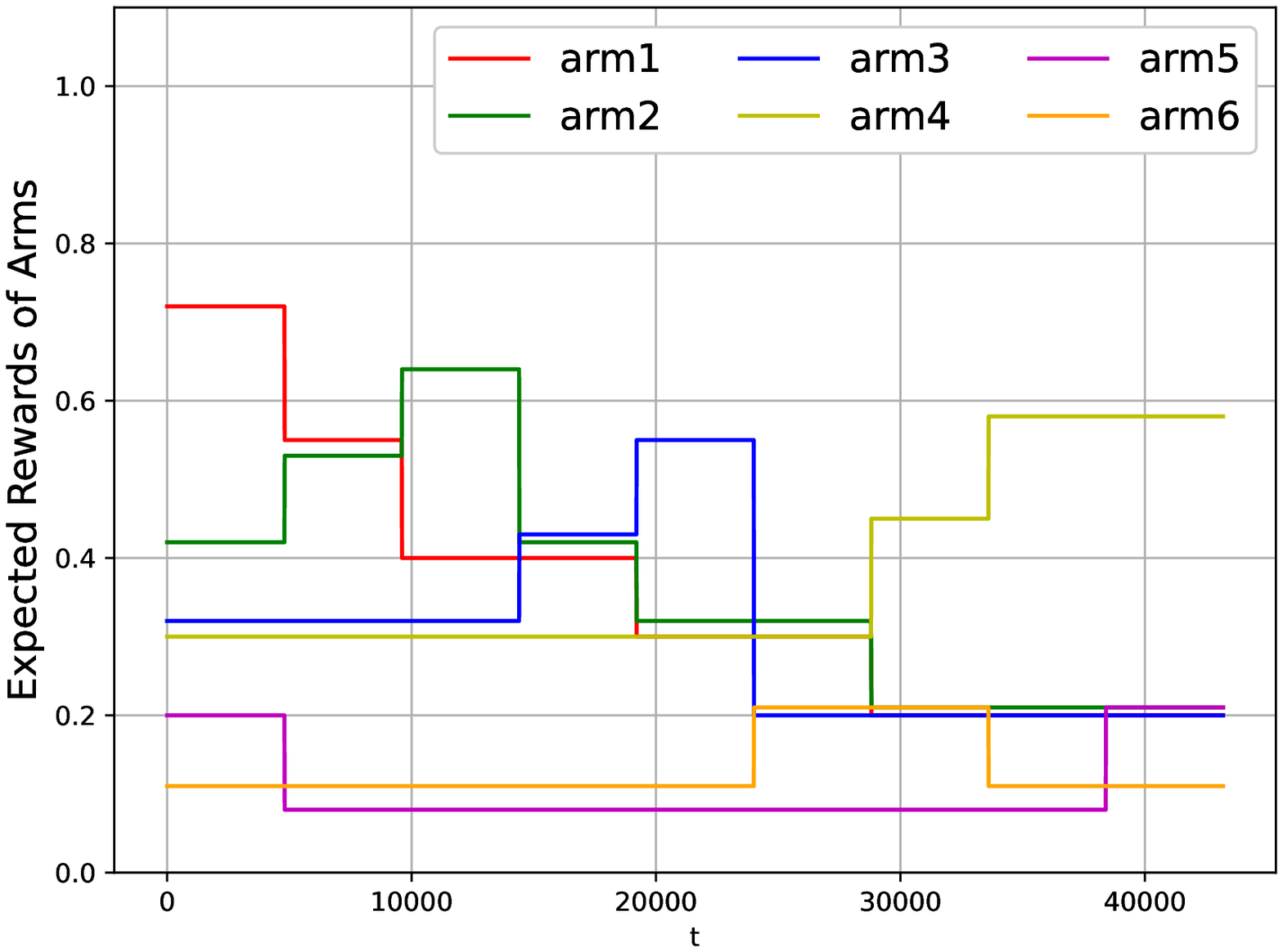}  
  \vspace{-5pt}
  \caption{Average rewards of arms.}
  \label{fig:e2-env}
\end{subfigure}
\caption{Setting of Experiment II (Yahoo! dataset).}
\label{fig:e2}
\end{figure}
\textbf{Figure~\ref{fig:e2-rgt}} shows the average regret of all algorithms while \textbf{Figure~\ref{fig:e2-cpd}} depicts the change point detection signals.

In \textbf{Figure~\ref{fig:e2-rgt}}, RBO-Coop-UCB and GLR-Coop-UCB have the lowest regret, while GLR-UCB has a regret lower than RBO-UCB. The reason is explainable using \textbf{Figure~\ref{fig:e2-cpd}} and Table~\ref{tab:cpd}. Algorithms with RBOCPD are more sensitive in detection; As such, they can detect change points faster. Besides, our proposed cooperation mechanism reduces false alarms while maintaining other performance indicators, such as correct detections and delays, at a comparable level. RBO-Coop detects more change points than GLR-Coop and guarantees fewer false alarms. In addition, according to the arms’ reward distributions, there are four change points in which the optimal arm remains fixed ($t=4800,19200, 33600, 38400$); thus, a restart at these steps increases the regret. That observation explains the similar regret performance of RBO-Coop-UCB and GLR-Coop-UCB despite the former detecting more change points. Besides, according to \textbf{Figure~\ref{fig:e2-rgt}}, the algorithms with cooperation (solid lines) always have lower regret than the algorithms without cooperation (dashed lines), which proves our designed cooperation mechanism enhances the performance.
\begin{figure}[!htp]
  \centering
  \vspace{-5pt}
  \includegraphics[width=0.75\linewidth]{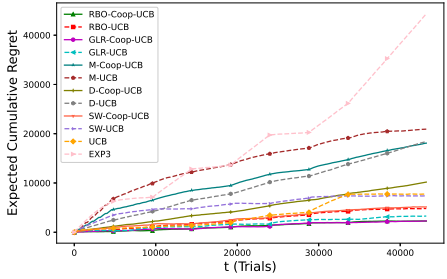}
  \caption{Expected cumulative regret for different algorithms using the pre-processed Yahoo! Dataset.}
  \label{fig:e2-rgt}
  \vspace{-15pt}
\end{figure} 
\begin{figure}[!ht]
\begin{subfigure}{.23\textwidth}
  \centering
  \includegraphics[width=.9\linewidth]{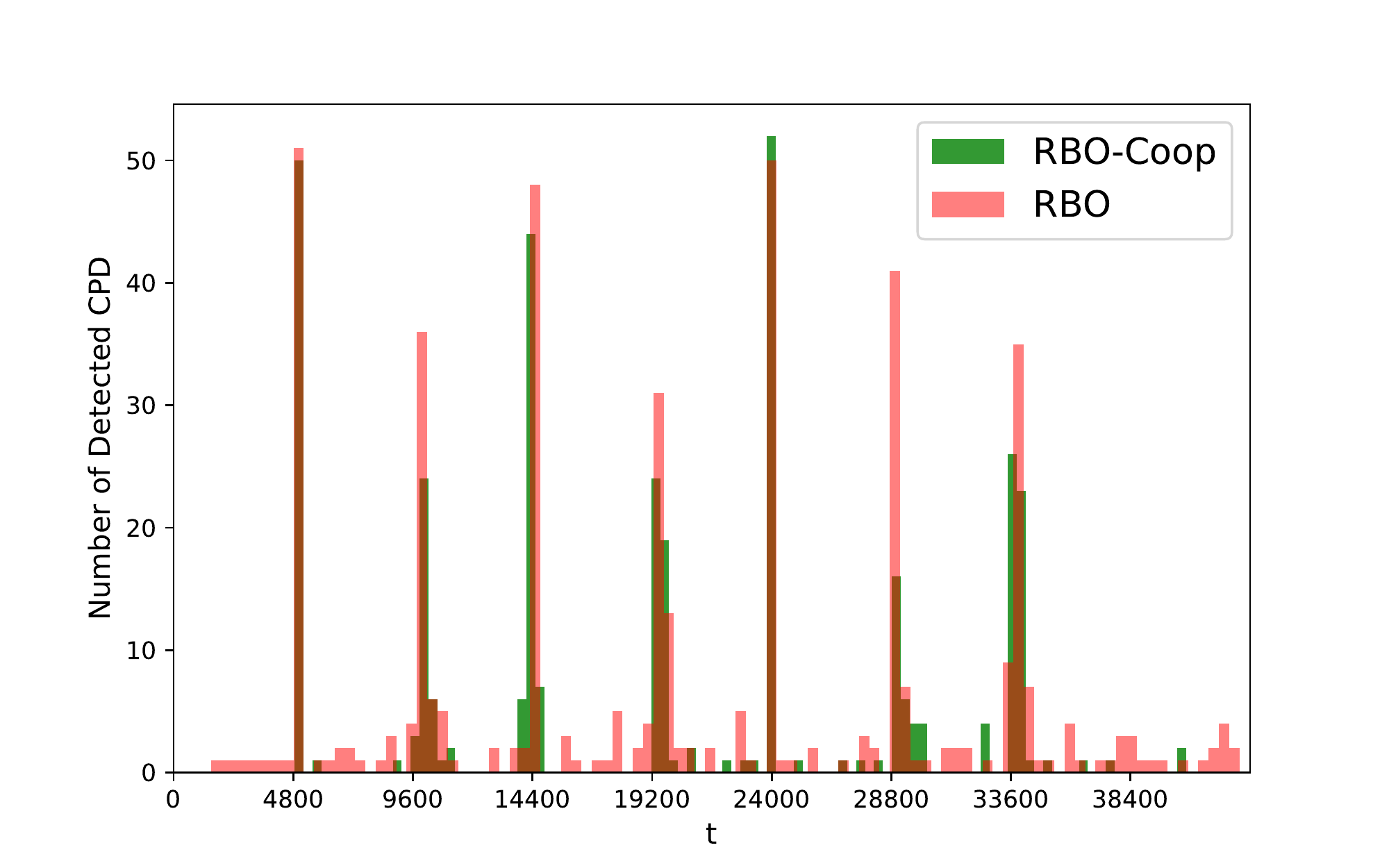}
 \caption{Change points detected by RBO-Coop and RBO.}
 \label{fig:e2-rbolog}
\end{subfigure}
\begin{subfigure}{.23\textwidth}
  \centering
  \includegraphics[width=.9\linewidth]{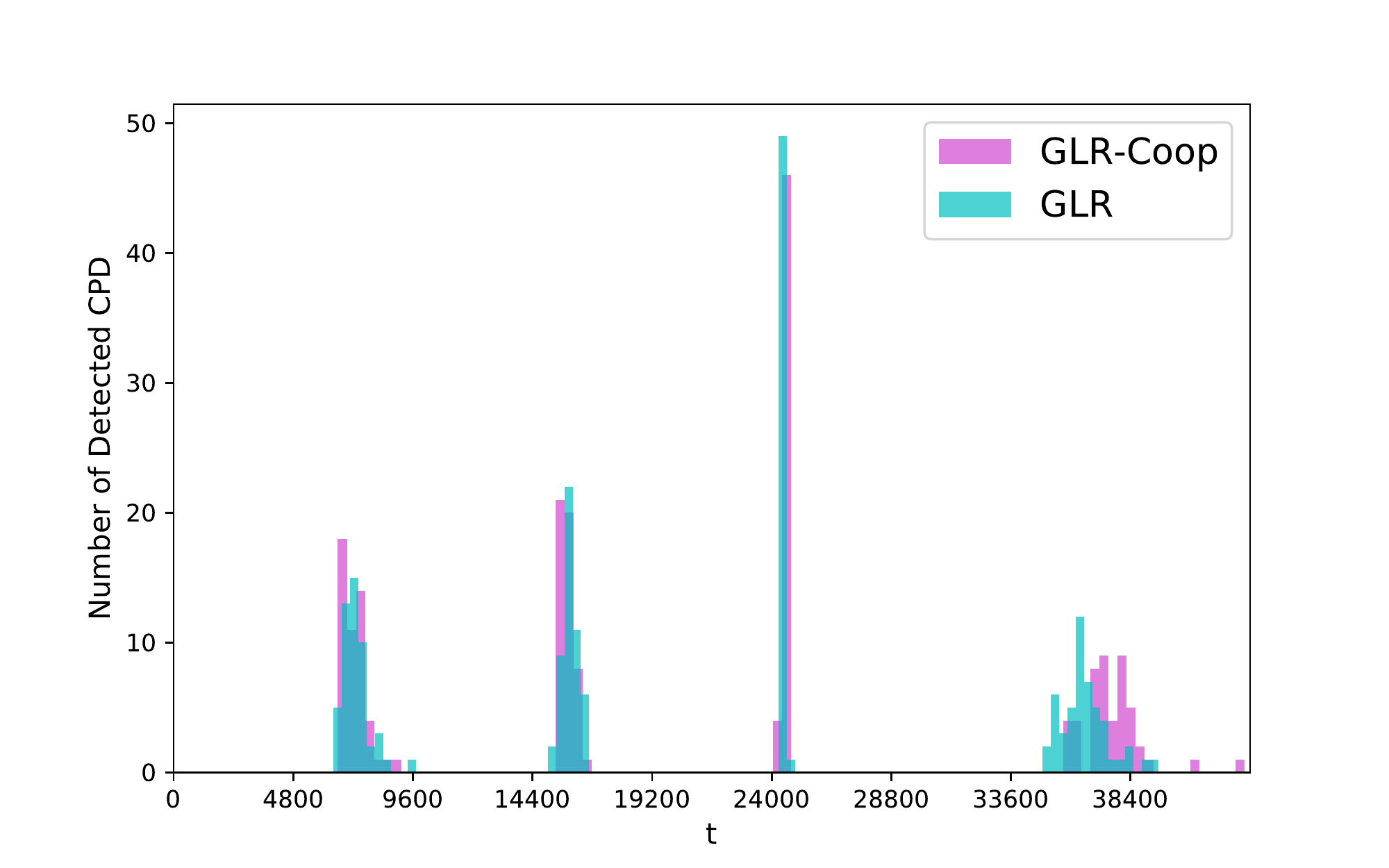}  
  \caption{Change points detected by GLR-Coop and GLR.}
  \label{fig:e2-glrlog}
\end{subfigure}
\newline
\begin{subfigure}{.23\textwidth}
  \centering
  \includegraphics[width=.9\linewidth]{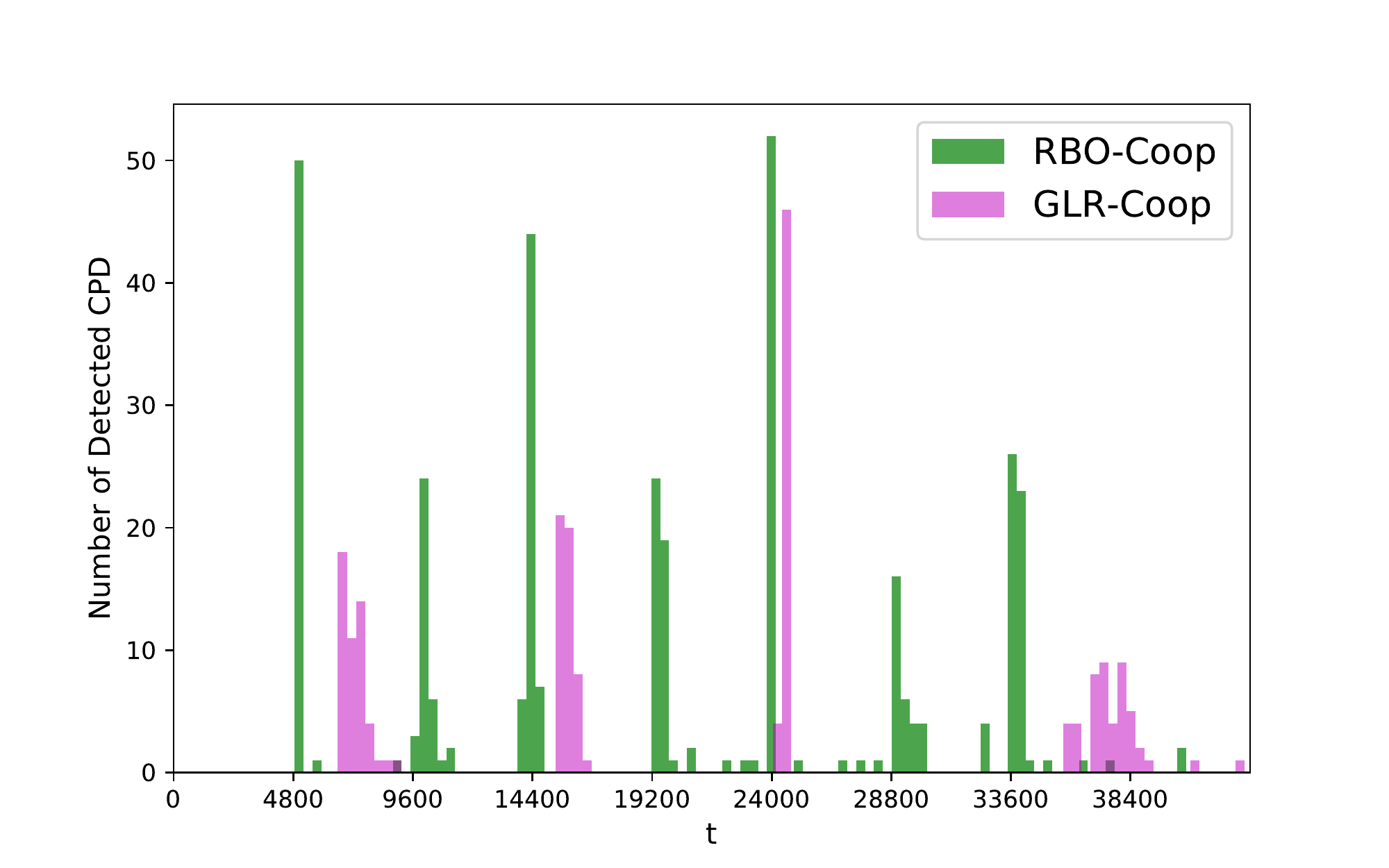}
  \caption{Change points detected by RBO-Coop and GLR-Coop.}
  \label{fig:e2-coop}
\end{subfigure}
\begin{subfigure}{.23\textwidth}
  \centering
  \includegraphics[width=.9\linewidth]{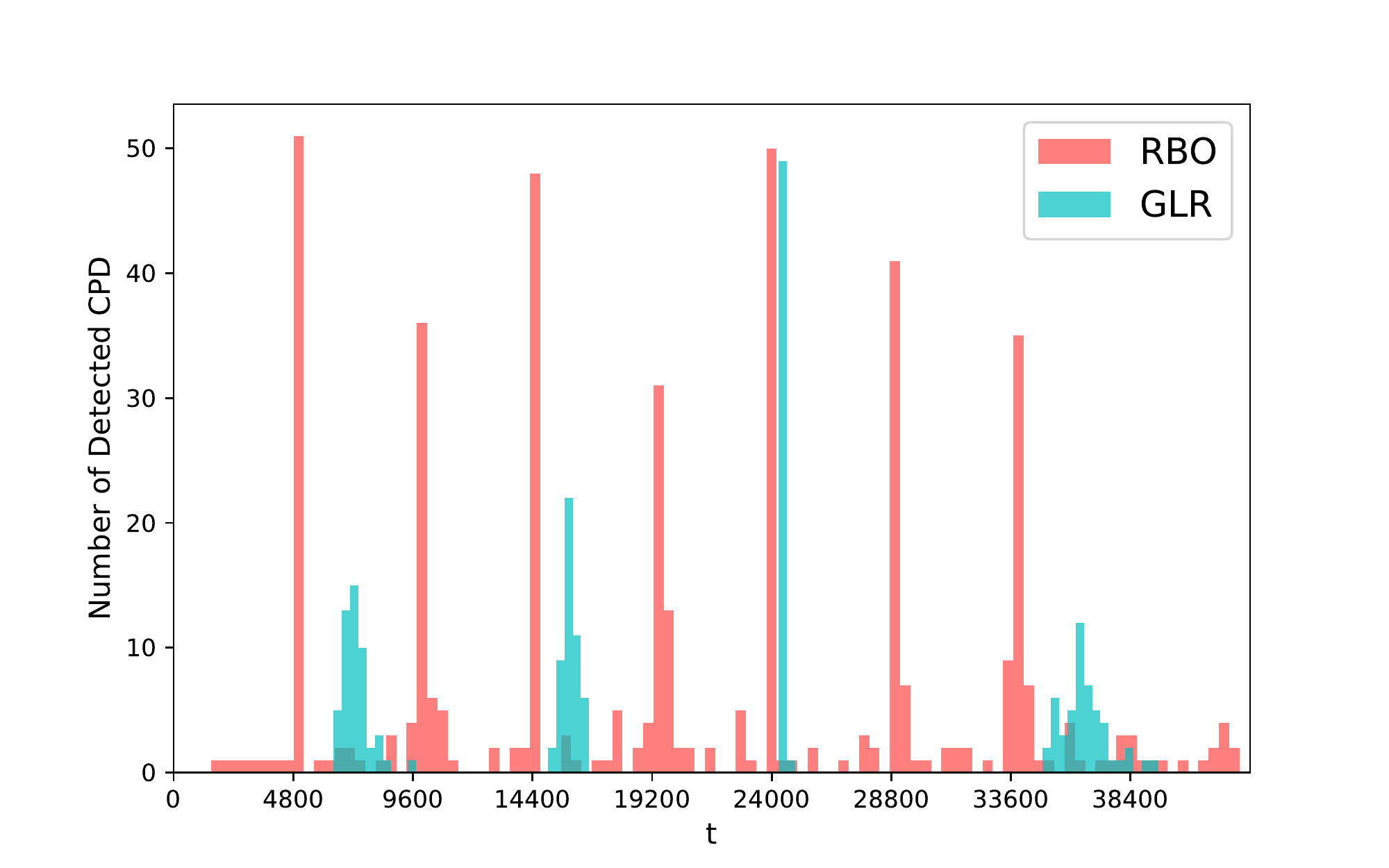}  
    \caption{Change points detected by RBO and GLR.}
    \label{fig:e2-nocoop}
\end{subfigure}
\caption{Change point detection of different algorithms on the Yahoo! Dataset.}
\label{fig:e2-cpd}
\end{figure}
\subsubsection{Digital Marketing Dataset} 
In \cite{cao2019nearly}, the authors preprocess the dataset. They show the corresponding arm reward distribution in Figure 5 of \cite{cao2019nearly}. Unlike the Yahoo! dataset, where the arms distribution is piecewise-Bernoulli, here, we only assume that the distribution is bounded within $[0,1]$. Similar to the previous experiment, we implement the dataset with a network with $K = 7$ agents as shown in \textbf{Figure~\ref{fig:e4-network}}. There are 12 piecewise-stationary segments as illustrated in \textbf{Figure~\ref{fig:e4-env}}. 
\begin{figure}[!ht]
\begin{subfigure}{.23\textwidth}
  \centering
  \includegraphics[width=.55\linewidth]{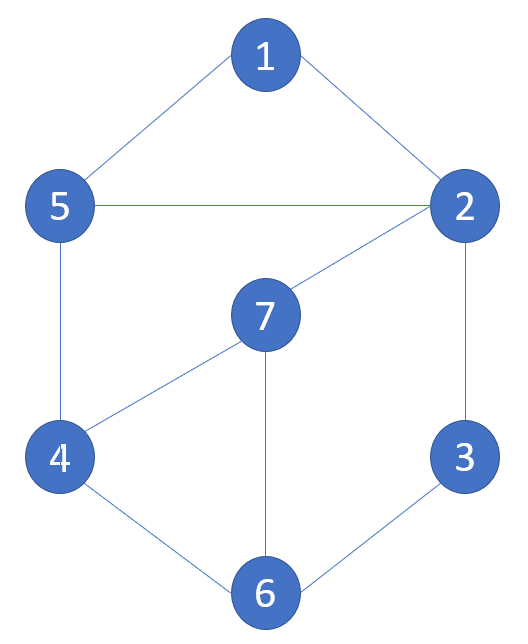}
  \caption{Observation network.}
  \label{fig:e4-network}
\end{subfigure}
\begin{subfigure}{.24\textwidth}
  \centering
  \vspace{-10pt}
  \includegraphics[width=.98\linewidth]{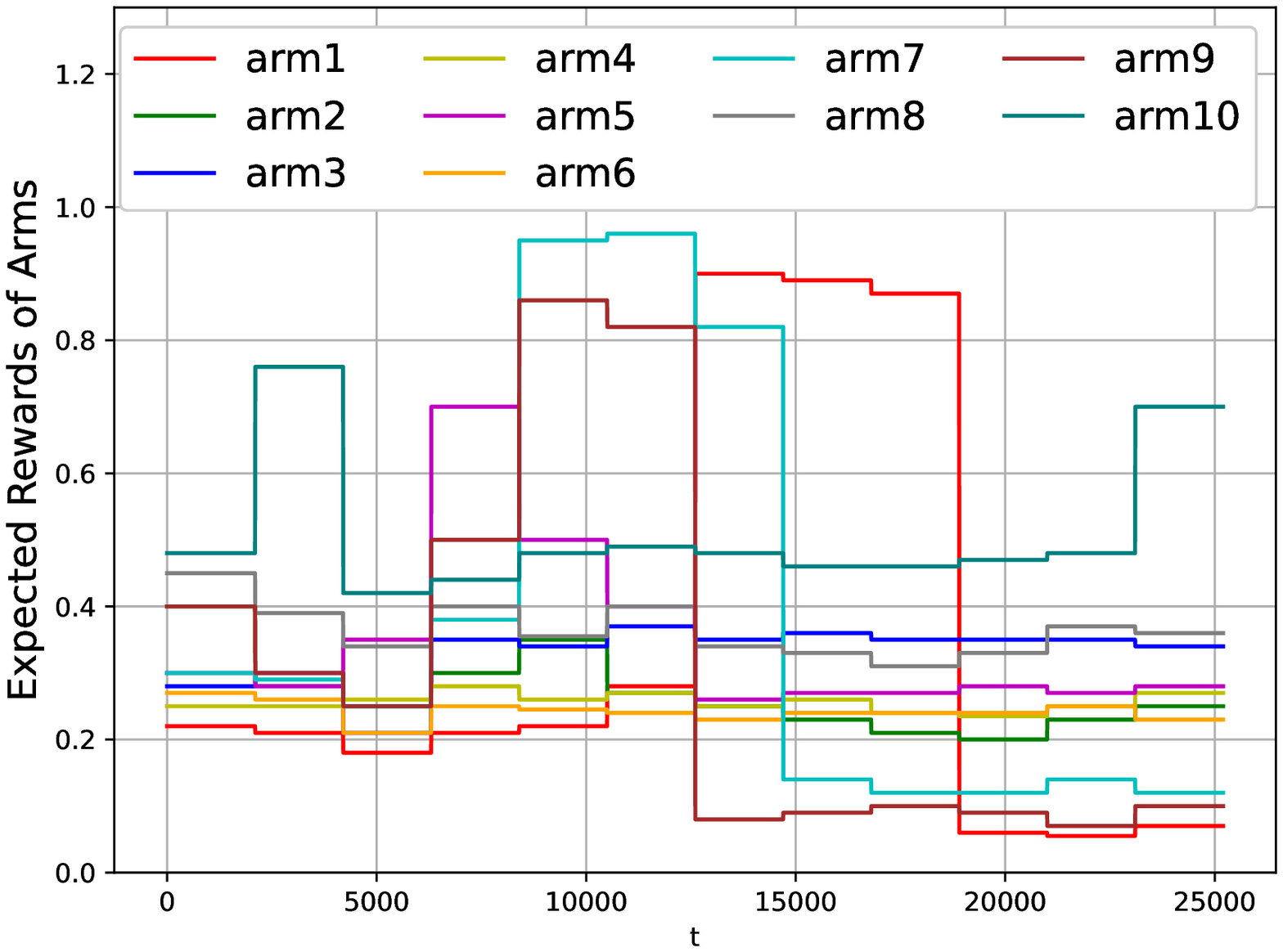}  
  \vspace{-15pt}
  \caption{Average rewards of arms.}
  \label{fig:e4-env}
\end{subfigure}
\caption{Setting of Experiment III (Digital Marketing dataset).}
\label{fig:e4}
\end{figure}

\textbf{Figure~\ref{fig:e4-rgt}} shows the average regret of different policies and \textbf{Figure~\ref{fig:e4-cpd}} demonstrates the performance of different change point detectors.
\begin{figure}[!htp]
  \centering
  \vspace{-15pt}
  \includegraphics[width=0.75\linewidth]{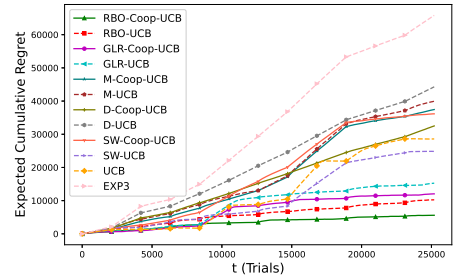}
  \caption{Expected cumulative regret for different algorithms using the Digital Marketing Dataset.}
  \label{fig:e4-rgt}
  \vspace{-15pt}
\end{figure} 
From \textbf{Figure~\ref{fig:e4-rgt}}, RBO-Coop-UCB has the lowest regret. \textbf{Figure~\ref{fig:e4-coop}} shows the detection performance. Accordingly, RBO-Coop performs better than GLR-Coop due to more sensitivity and lower delay.

According to the theoretical analysis in Remark~\ref{rem:comp-glr-rbo}, RBOCPD suffers a higher false alarm rate. In Experiment I, the regret of RBO-UCB is higher than GLR because of more frequent false alarms. In Experiment II, given a lower false alarm empirically according to Table~\ref{tab:cpd}, the reason for higher regret of RBO-UCB than GLR-UCB is restarting all of the arms, even if the optimal one remains unchanged. As shown in \textbf{Figure~\ref{fig:e4-cpd}} and Table~\ref{tab:cpd}, the false alarm of RBO in Experiment III is comparable to that of GLR, whereas its detection rate is much higher. Therefore, RBO has the second-best performance in this experiment. GLR does not perform at the expected level because it is not as robust as RBO when the gap in arms' rewards between change points is small. Finally, except for SW-UCB, the other algorithms with cooperation have lower regret than the version without, which is the same as the results that appear in \textbf{Figure~\ref{fig:e1-rgt}}.
\begin{figure}[!ht]
\begin{subfigure}{.23\textwidth}
  \centering
  \includegraphics[width=.9\linewidth]{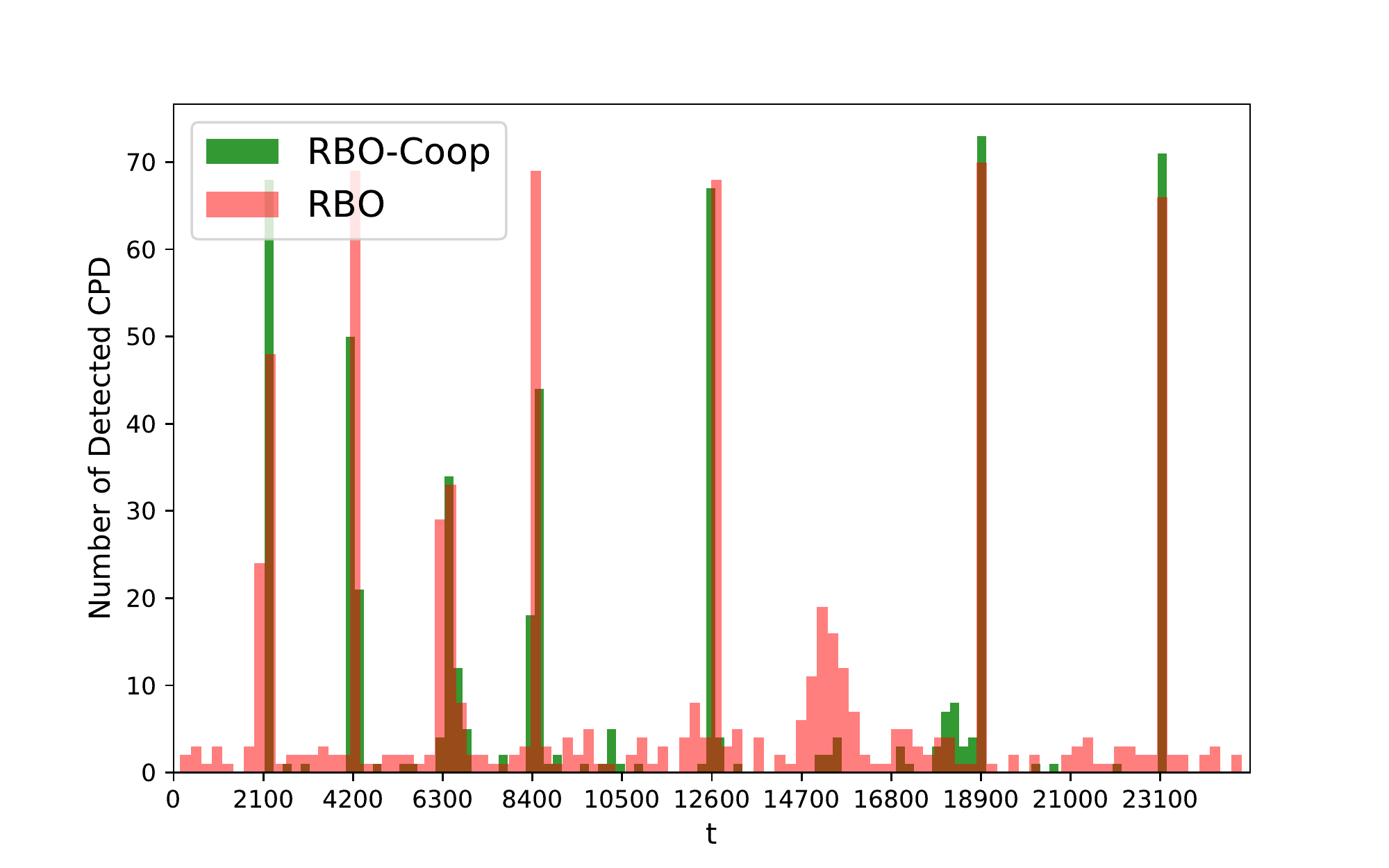}
  \caption{Change points detected by RBO-Coop and RBO.}
  \label{fig:e4-rbo}
\end{subfigure}
\begin{subfigure}{.23\textwidth}
  \centering
  \includegraphics[width=.9\linewidth]{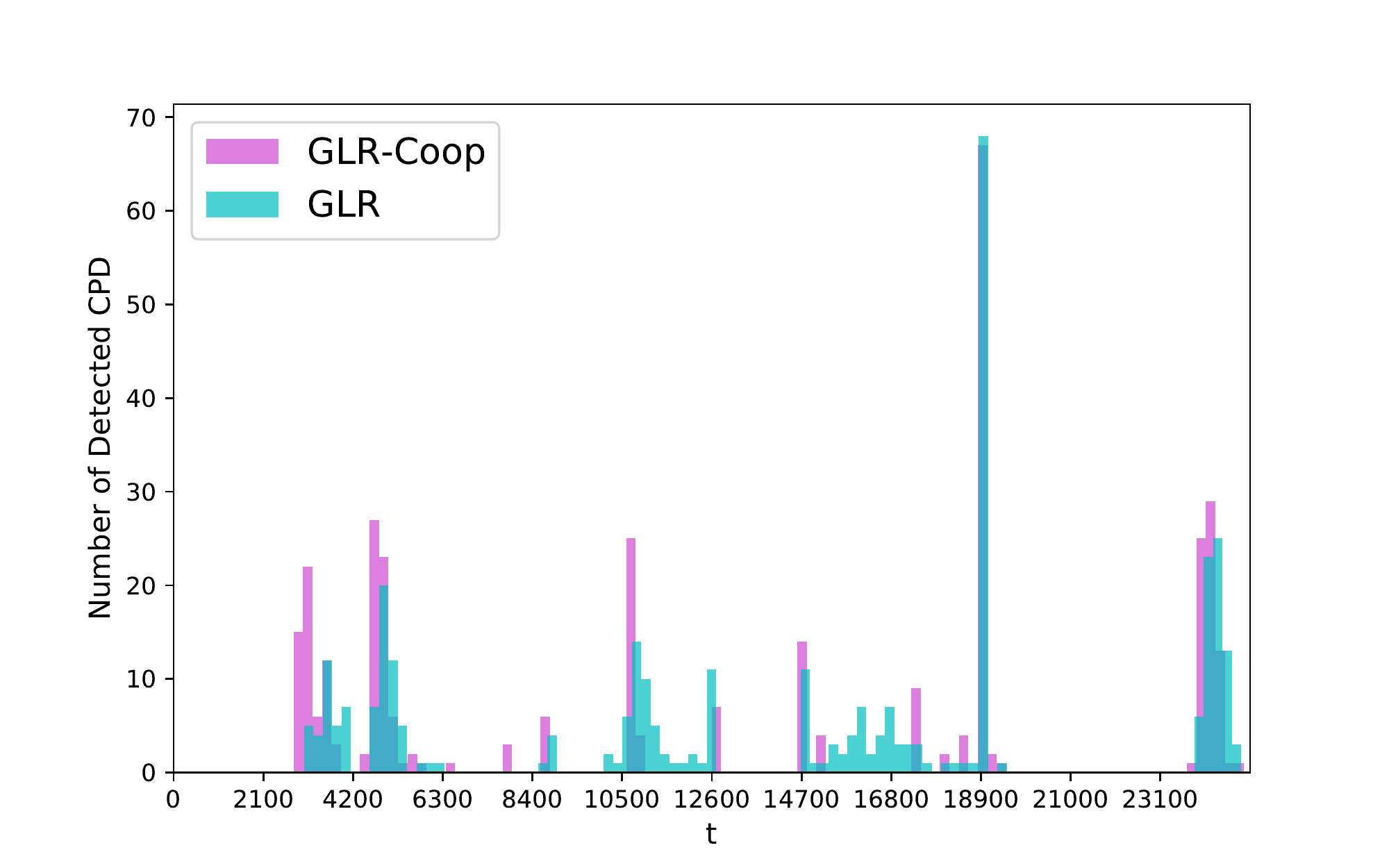}  
  \caption{Change points detected by GLR-Coop and GLR.}
  \label{fig:e4-glr}
\end{subfigure}
\newline
\begin{subfigure}{.23\textwidth}
  \centering
  \includegraphics[width=.9\linewidth]{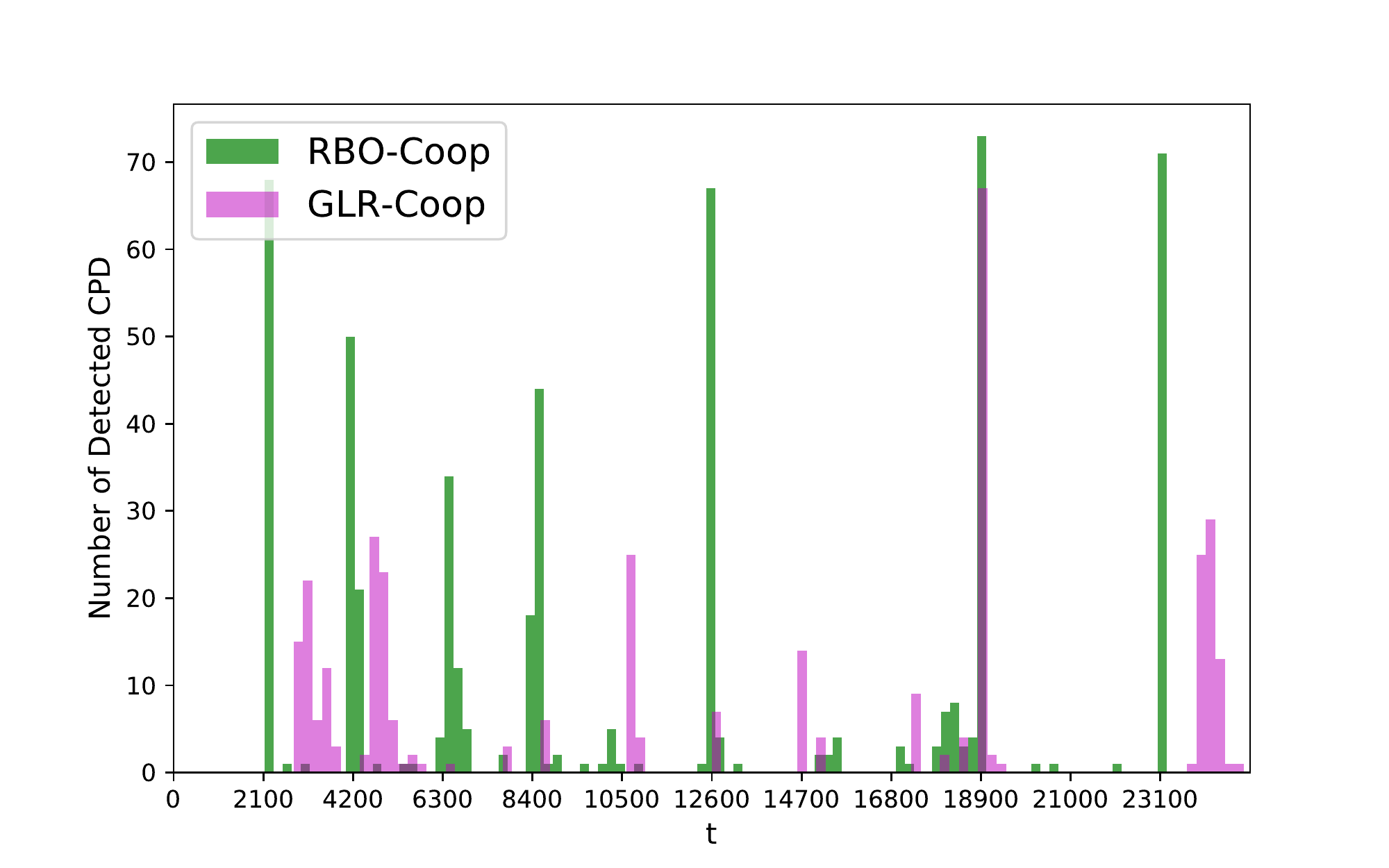}
  \caption{Change points detected by RBO-Coop and GLR-Coop.}
  \label{fig:e4-coop}
\end{subfigure}
\begin{subfigure}{.23\textwidth}
  \centering
  \includegraphics[width=.9\linewidth]{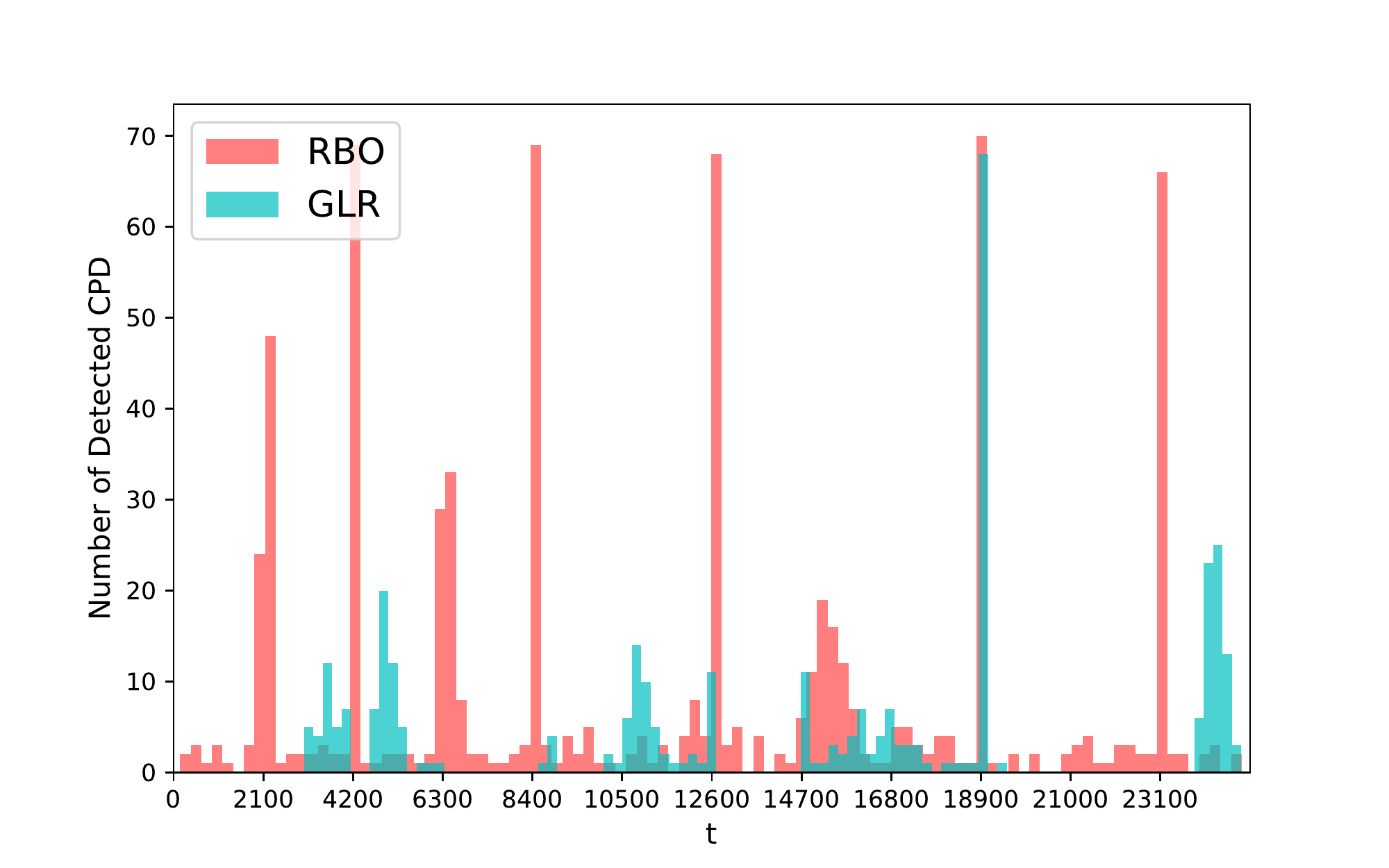}  
  \caption{Change points detected by GLR and RBO.}
  \label{fig:e4-ind}
\end{subfigure}
\caption{Change point detection of different algorithms on the  Digital Marketing Dataset.}
\label{fig:e4-cpd}
\end{figure}
\section{Conclusion} 
\label{sec:concl}
We propose a decision-making policy for the multi-agent multi-armed bandit problem in a piecewise-stationary environment. The process involves the seminal UCB framework combined with an information-sharing mechanism and a change point detector based on the Bayesian strategy. We prove an upper bound for the group regret of our proposal. Numerical analysis shows the superior performance of our method compared to the state-of-art research. Future research directions include improving the UCB strategy by personalizing it for each agent. For example, it is beneficial to consider heterogeneous agents' preferences. Besides, one might allow for different agent types and influence models and information-sharing under a directed graph model.
\appendix
\subsection{Performance Guarantees of RBOCPD} 
\label{app:rbocpd}
\begin{lemma}[False alarm rate]
\label{the:far}
Assume that $\boldsymbol{X}_{r:t} \sim \mathcal{B}(\mu)$. Let $\alpha > 1$. If $\eta_{r,s,t}$ is small enough such that \cite{alami2020restarted}
{\footnotesize
\begin{align}
   & \forall t \in [r,\nu_n), s \in (r,t]: \notag \\
   & \eta_{r,s,t} <    \frac{\sqrt{n_{r:s-1} \times n_{s:t}}}{10(n_{r:t}+1)} \times \notag \\
   & (\frac{\log 2 \log^4(\alpha)  \delta^2}{4n_{r:t} \log^2(n_{r:t}) \log(\alpha n_{r:s-1}) \log(n_{r:s-1}) \log(\alpha n_{s:t}) \log(n_{s:t})})^{\alpha} \notag 
\end{align}
}%
then with probability higher than $1-\delta$, no false alarm occurs in the interval $[r,\nu_n)$:
\begin{gather}
    \mathbb{P}_{\theta} \{\exists t \in [r,\nu_n),  \textbf{Restart}_{r:t} = 1\} \leq \delta. 
\end{gather}
\end{lemma}
\begin{definition}[Relative gap $\Delta_{r,s,t}$]
Let $\Delta \in [0,1]$. The relative gap $\Delta_{r,s,t}$ for the forecaster $s$ at time $t$ takes the following form (depending on the position of $s$) \cite{alami2020restarted}:
\begin{gather}
    \Delta_{r,s,t} = (\frac{n_{r:\nu_n-1}}{n_{r:s-1}} \mathbbm{1}\{\nu_n \leq s \leq t\} + \frac{n_{\nu_n:t}}{n_{s:t}}\mathbbm{1}\{s < \nu_n\})\Delta \notag 
\end{gather}
\end{definition}
\begin{lemma}[Detection delay]
\label{the:dd}
Let $\boldsymbol{x}_{r:\nu_n-1} \sim \mathcal{B}(\mu_1), \boldsymbol{x}_{\nu_n:t} \sim \mathcal{B}(\mu_2)$ and $f_{r,s,t} = \log n_{r:s}+ \log n_{s:t+1} - \frac{1}{2} \log n_{r:t} + \frac{9}{8}$. Also, $\Delta = |\mu_1 - \mu_2|$ is the change point gap. If $\eta_{r,s,t}$ is large enough such that \cite{alami2020restarted}
\begin{gather}
    \eta_{r,s,t} > \exp (-2n_{r,s-1}(\Delta_{r,s,t} - \mathcal{C}_{r,s,t,\delta})^2 + f_{r,s,t}),
\end{gather}
Then the change point $\nu_n$ is detected (with a probability at least $1-\delta$) with a delay not exceeding $\mathcal{D}_{\Delta,r,\nu_n}$ such that
\begin{align}
    \mathcal{D}_{\Delta,r,\nu_n} = \min &\{d \in \mathbb{N}^*: d > \frac{(1-\frac{\mathcal{C}_{r,\nu_n,d+\nu_n-1,\delta}}{\Delta})^{-2}}{2\Delta^2} \notag \\
    &\times \frac{-\log \eta_{r,\nu_n,d+\nu_n-1}+f_{r,\nu_n,d+\nu_n-1}}{1+\frac{\log \eta_{r,\nu_n,d+\nu_n-1}-f_{r,\nu_n,d+\nu_n-1}}{2n_{r,\nu_n-1}(\Delta-\mathcal{C}_{r,\nu_n,d+\nu_n-1,\delta})^2}}\}, \notag
\end{align}
where 
\begin{align}
    &\mathcal{C}_{r,s,t,\delta} = \frac{\sqrt{2}}{2}(\sqrt{\frac{1+\frac{1}{n_{r:s-1}}}{n_{r:s-1}}\log(\frac{2\sqrt{n_{r:s}}}{\delta})} \notag \\
    &\quad + \sqrt{\frac{1+\frac{1}{n_{s:t}}}{n_{s:t}}\log(\frac{2n_{r:t}\sqrt{n_{s:t}+1}\log^2(n_{r:t})}{\log(2)\delta})}) \label{eq:crst}.
\end{align}
\end{lemma}
\begin{definition}[Detectable change points]
\label{def:dcp}
A change point $\nu_n$ is $(\varepsilon,d)$-detectable (with probability $1-\delta$) with respect to the sequence $(\nu_n)_n$ for the delay function $\mathcal{D}$ if \cite{maillard2019sequential}
\begin{gather}
    \mathcal{D}_{\Delta,(\nu_{n-1}+1+\varepsilon l_{n-1}),\nu_n} < l_n,
\end{gather}
where $l_{n-1} = \nu_n -\nu_{n-1}$. That is, using only a fraction $1-\varepsilon$ of the $l_{n-1}$-many observations available between the previous change point and $\nu_n + 1$, the change can be detected at a time not exceeding the next change point.
\end{definition}
%
\begin{remark}\label{rem:comp-glr-rbo}
Compared to other change point detection methods used in MAB algorithms, RBOCPD has the same advantages as GLRCPD, such as few hyperparameters (in RBOCPD, only one hyperparameter $\eta_{r,s,t}$). Besides, it is robust to the lack of prior knowledge. Based on the theoretical analysis in Appendix E, F of \cite{besson2019generalized} and our work, RBOCPD has lower computational complexity and higher robustness against small gap cases \cite{alami2020restarted} while slightly higher false alarm rate than GLRCPD given the same hyperparameter (In GLRCPD the false alarm rate $\delta$ can be directly selected as hyperparameter while according to Lemma~\ref{the:far}, in RBOCPD the false alarm rate $\delta \sim o(\eta_{r,s,t}^{\frac{1}{2\alpha}})$). 
\end{remark}
\subsection{Proof of Lemma~\ref{lem:rgt-sta}} 
\label{app:rgt-sta}
First, we investigate the improvement in the false alarm as a result of cooperation in making the restart decision in each agent's neighborhood, as the following Lemma~\ref{lem:cfar} states.
\begin{lemma}[False alarm rate under cooperation]
\label{lem:cfar}
Define $\sigma$ as the cooperative false alarm rate where each agent implements a change point detector. The false alarm rate in a information-sharing mechanism can be concluded as
\begin{align}
    & P(\exists t \in [r,\tau_c): \sum_{j \in \mathcal{N}_k} \mathbbm{1}\{\exists i \in [N_{t-d}^{j,m}, N_t^{j,m}], r_i^{j,m} > 0\}  \geq \lceil \frac{\eta_k}{2} \rceil) \notag \\
    &= \sum_{j=0}^{\lceil \frac{\eta_k}{2} \rceil}\binom{\eta_k}{j}\delta^{\eta_k - j}(1-\delta)^j = \sigma,
\end{align}
where $\delta$ is the false alarm rate of every single agent, and $\eta_k$ is the number of neighbors of agent $k$ (degrees of node $k$ in graph $\mathcal{G}$).
\end{lemma}
\begin{proof}
For each agent, the false alarm rate is $\delta$. Therefore, when cooperating, the probability that there are $j$ agents having false alarms at the same time is $\binom{\eta_k}{j}\delta^j(1-\delta)^{\eta_k - j}$. Based on the information-sharing mechanism, the algorithm has a false alarm only when more than half of the neighbor agents report a false alarm. Therefore, considering all possible scenarios, the false alarm rate yields $\sigma = \sum_{j=\lceil \frac{\eta_k}{2} \rceil}^{\eta_k}\binom{\eta_k}{j}\delta^j(1-\delta)^{\eta_k - j}=\sum_{j=0}^{\lceil \frac{\eta_k}{2} \rceil}\binom{\eta_k}{j}\delta^{\eta_k - j}(1-\delta)^j$. Besides, a $\delta << 1$ satisfies $\sum_{j=0}^{\lceil \frac{\eta_k}{2} \rceil}\binom{\eta_k}{j}\delta^{\eta_k - j-1}(1-\delta)^j \leq 1$, $\sigma < \delta$, which implies that the cooperation in change point detectors reduces the false alarm rate.  Different number of neighbors $\eta_k$ implies different value of $\sigma$ and the higher $\eta_k$ is, the smaller $\sigma$ is. 
\end{proof}
Now we are in the position to prove Lemma~\ref{lem:rgt-sta}.\\
The expected regret can be written as
\begin{gather}
    \mathcal{R}_T^k \leq \Delta_1^* \sum_{m \in \mathcal{M}}N_T^{k,m},
\end{gather}
where $\Delta_1^*$ is the expected largest reward difference. Besides, $\sum_{m\in \mathcal{M}} N_T^{k,m} = \sum_{m\in \mathcal{M}} N_T^{k,m} \mathbbm{1}\{\tau_1 \leq T\} + \sum_{m\in \mathcal{M}} N_T^{k,m} \mathbbm{1}\{\tau_1 > T\}$ and we have,
\begin{align}
     \sum_{m\in \mathcal{M}} N_T^{k,m} \mathbbm{1}\{\tau_1 > T\} &\leq  p T + \sum_{m\in\mathcal{M}} \sum_{t=1}^T \mathbbm{1}\{I_t^k \neq i^*_t, N_T^{k,m} < l\} \notag \\
    &\quad + \sum_{m\in\mathcal{M}} \sum_{t=1}^T \mathbbm{1}\{I_t^k \neq i^*_t, N_T^{k,m} > l\}, \notag 
\end{align}
where $pT$ refers to the regret of the forced exploration. Considering the worst situation, where all $l$ steps are not optimal, $\sum_{m\in\mathcal{M}} \sum_{t=1}^T \mathbbm{1}\{I_t^k \neq i^*_t, N_T^{k,m} > l\} \leq Ml$.  Besides, $\{I_t^k \neq i^*_t, N_T^{k,m} > l\} \subseteq \{\bar{X}_t^{k,I_t^k} \geq \mu_t^{k,I_t^k} + C_t^{k,I_t^k}\} \cup  \{\bar{X}_t^{k,i^*_t} \geq \mu_t^{k,i^*_t} - C_t^{k,i^*_t}\} \cup  \{\mu_t^{k,i^*_t}-\mu_t^{k,I_t^k} \leq 2C_t^{k,I_t^k}, N_T^{k,m} > l \} $ \cite{auer2002finite}. Denote $\Delta_1^{\min}$ as the expected lowest reward difference, then for $l = \lceil \frac{8 \xi \log T}{ (\Delta_1^{\min})^2} \rceil$, $\{\mu_t^{k,i^*_t} -\mu_t^{k,I_t^k} \leq 2C_t^{k,I_t^k}, N_T^{k,m} > l \} = \emptyset $, because $\mu_t^{k,i^*_t} -\mu_t^{k,I_t^k}  - 2C_t^{k,I_t^k} = \mu_{i^*_t,t} -\mu_{I_t,t} - 2\sqrt{\frac{\xi (\alpha+1)\log t}{N_t^{k,I_t^k}}} \geq \Delta_1^{\min} - 2\sqrt{\frac{\xi (\alpha+1)\log t}{\frac{8 \xi \log T}{ (\Delta_1^{\min})^2}}} \geq \Delta_1^{\min} (1 - \sqrt{\frac{(\alpha+1)\log t}{2\log T}}) \geq 0$, $\forall I_t^k \in \mathcal{M}$,.
\begin{fact}[Chernoff-Hoeffding bound]
\label{fact:chb}
Let $X_1,X_2, \ldots, X_n$ be random variables with common range $[0,1]$ and such that $\mathbb{E}[X_t|X_1, \ldots, X_{t-1}] = \mu$. Let $S_n = X_1+X_2+ \ldots+X_n$, then for all $a \geq 0$,
$P(S_n \geq n\mu+a) \leq e^{-2a^2/n}$ and $P(S_n \leq n\mu-a) \leq e^{-2a^2/n}$.
\end{fact}
According to Fact~\ref{fact:chb}, $P(\bar{X}_t^{k,I_t^k} \geq \mu_t^{k,I_t} + C_t^{k,I_t^k}) \leq e^{-2\xi(\alpha+1)\log t} = t^{-2\xi(\alpha+1)}$, and $P(\bar{X}_t^{k,i^*_t} \geq \mu_t^{k,i^*_t} - C_t^{k,i^*_t}) \leq t^{-2\xi(\alpha+1)}$. Choose $\xi > \frac{2}{\alpha+1}$, $t^{-2\xi(\alpha+1)} \leq t^{-4}$. Therefore, 
\begin{align}
    \sum_{m\in \mathcal{M}} N_T^{k,m} &\leq TP(\tau_1 \leq T) + pT + M(\lceil \frac{8 \xi \log T}{(\Delta_1^{\min})^2} \rceil + 1+\frac{\pi^2}{3}), \notag \\
    & \leq M\sigma T + pT + M(\lceil \frac{8 \xi \log T}{(\Delta_1^{\min})^2} \rceil + 1+\frac{\pi^2}{3}), \notag 
\end{align}
and the expected regret follows as
\begin{align}
    \mathcal{R}_T^k & \leq \Delta_1^*[M\sigma T + pT + M\lceil \frac{8 \xi\log T}{(\Delta_1^{\min})^2} \rceil + M(1+\frac{\pi^2}{3})], \notag \\
    &= \Delta_1^*M\sigma T + \Delta_1^*pT + \tilde{C}_1^k,
\end{align}
where $\tilde{C}_1^k = \Delta_1^*[M\lceil \frac{8 \xi\log T}{(\Delta_1^{\min})^2} \rceil + M(1+\frac{\pi^2}{3})]$. The first term is due to the false alarm probability in RBOCPD, the second term is because of the uniform exploration, and the last term is caused by UCB exploration. 
\subsection{Proof of Lemma~\ref{lem:fap-sc}} 
\label{app:fap-sc}
\begin{proof}
In the stationary scenario, based on the theoretical false alarm rate of RBOCPD stated by Theorem~\ref{the:far} and Lemma~\ref{lem:cfar}, $P(\tau_1^{k,m} \leq \tau) \leq \sigma$ for every $k \in \mathcal{K}$ and every $m \in \mathcal{M}$. Therefore, $P(\tau_1^k \leq T) \leq \sum_{m=1}^M P(\tau_1^{k,m} \leq \tau) \leq  M\sigma$. 
\end{proof}
\subsection{Proof of Lemma~\ref{lem:p-delay}} 
\label{app:p-delay}
\begin{proof}
(a) is from Lemma~\ref{lem:fap-sc}.\\
According to Fact~\ref{fact:gap} and Definition~\ref{def:dcp}, between change points $\nu_n$ and $\nu_{n-1}$, each arm has been selected at least $2 \frac{p}{M}\times (\frac{M}{p}\mathcal{D}_{\Delta,(\nu_{n-1}+d_{n-1}^{k,m}),\nu_n}+ \frac{M}{p}) = 2\mathcal{D}_{\Delta,(\nu_{n-1}+d_{n-1}^{k,m}),\nu_n} + 2$ times. Given good event $\mathcal{C}_{n-1}^k$, we have $\tau_{n-1} \leq \nu_{n-1} + d_{n-1}^k$. Besides, $\mathcal{D}_{\Delta,(\nu_{n-1}+d_{n-1}^{k,m}),\nu_n}$ is detectable with probability $1 - \delta$ based on Definition~\ref{def:dcp} within length $\nu_n - \tau_{n-1}$ because $\nu_n - \tau_{n-1} \geq \nu_n - (\nu_{n-1}+d_{n-1}) \geq 2 \max(d_n,d_{n-1}) - d_{n-1} \geq \max(d_n,d_{n-1})$, where the first inequality results from the good event $\mathcal{C}_{n-1}^k$. Therefore, one can conclude that $P(\tau_n^k \leq \nu_n + d_n^k | \mathcal{C}_{(n-1)}^k) \geq 1 - \delta$, which is equivalent to $P(\tau_n^k > \nu_n + d_n^k | \mathcal{C}_{(n-1)}^k) \leq  \delta$. That completes the proof.
\end{proof}
\subsection{Proof of Theorem~\ref{the:rgt}} 
\label{app:the-rgt}
The regret in the piecewise-stationary environment is a collection of regrets of good events and bad events. The regret of good events is because of UCB exploration, whereas that of bad events results from false alarms and long detection delays. The derivation of regret upper bound is similar to Lemma~\ref{lem:rgt-sta},  as the regret of bad events can be bounded using the theoretical guarantee of the RBOCPD described in Appendix~\ref{app:rbocpd}. Before proceeding to the detailed proof, we state the following fact:
\begin{fact}
\label{fact:gap}
For every pair of instants $s \leq t \in \mathbb{N}^*$ between two restarts on arm $m$, it holds $n_t^{k,m} - n_s^{k,m} \geq 2 \lfloor \frac{p}{M}(t-s) \rfloor$. 
\end{fact}
%
Define the good event $\mathcal{F}_{n}^k = \{\tau_n > \nu_n \}$ and good event $\mathcal{T}_n^k = \{ \tau_n^k< \nu_n+d_n^k \}$. Recall the definition of the good event $\mathcal{C}_n^k$ in \eqref{eq:gec}. We note that $\mathcal{C}_n^k = \mathcal{F}_1^k \cap \mathcal{T}_1^k \cap \ldots \cap \mathcal{F}_n^k \cap \mathcal{T}_n^k$ is the intersection of the event sequence of $\mathcal{F}_n^k$ and $\mathcal{T}_n^k$ up to the $n$-th change point. We decompose the expected cumulative regret with respect to the event $\mathcal{F}_1^k$. That yields
\begin{align}
    \mathcal{R}_T^k = \mathbb{E}[R_T^k] &= \mathbb{E}[R_T^k\mathbb{I}(\mathcal{F}_1^k)] + \mathbb{E}[R_T^k\mathbb{I}(\widebar{\mathcal{F}}_1^k)] \notag \\
    & \leq \mathbb{E}[R_T^k\mathbb{I}(\mathcal{F}_1^k)] + T\Delta_{opt}^{max} P(\widebar{\mathcal{F}}_1^k) \notag \\
    & \leq \mathbb{E}[R_{\nu_1}^k\mathbb{I}(\mathcal{F}_1^k)] + \mathbb{E}[R_{T-\nu_1}^k] + T \Delta_{opt}^{max} M \sigma \notag \\
    & \leq \tilde{C}_1^k + \Delta_1^*p\nu_1 + T \Delta_{opt}^{max} M\sigma + \mathbb{E}[R_{T-\nu_1}^k], \notag 
\end{align}
where $\tilde{C}_1^k = \Delta_1^*[M\lceil \frac{8 \ln T}{(\Delta_1^{\min})^2} \rceil + M(1+\frac{\pi^2}{3})]$, which is concluded from Lemma~\ref{lem:rgt-sta}. By the law of total expectation,
\begin{align}
    \mathbb{E}[R_{T-\nu_1}^k] &\leq \mathbb{E}[R_{T-\nu_1}^k|\mathcal{F}_1^k \cap \mathcal{T}_1^k] + T\Delta_{opt}^{max}(1 - P(\mathcal{F}_1^k \cap \mathcal{T}_1^k)) \notag \\
    &= \mathbb{E}[R_{T-\nu_1}^k|\mathcal{F}_1^k \cap \mathcal{T}_1^k] + T\Delta_{opt}^{max} P(\widebar{\mathcal{F}}_1^k \cap \widebar{\mathcal{T}}_1^k) \notag \\
    & \leq \mathbb{E}[R_{T-\nu_1}^k|\mathcal{F}_1^k \cap \mathcal{T}_1^k]  + T\Delta_{opt}^{max}(M \sigma + \delta). 
    \label{eq:fal-ala}
\end{align}
Inequality \eqref{eq:fal-ala} results from Lemma~\ref{lem:p-delay}. Furthermore,
\begin{align}
    \mathbb{E}[R_{T-\nu_1}^k|\mathcal{F}_1^k \cap \mathcal{T}_1^k] &= \mathbb{E}[R_{T-\nu_1}^k|\mathcal{C}_1^k]. 
\end{align}
By combining the previous steps, we have 
\begin{gather}
    \mathcal{R}_T^k \leq \mathbb{E}[R_{T-\nu_1}^k|\mathcal{C}_1^k] + \tilde{C}_1^k + \Delta_1^*p\nu_1 + T\Delta_{opt}^{max}(2M \sigma + \delta). \notag 
\end{gather}
Similarly, 
\begin{align}
    \mathbb{E}[R_{T-\nu_1}^k|\mathcal{C}_1^k] &\leq \mathbb{E}[R_{T-\nu_1}^k\mathbb{I}(\mathcal{F}_2^k)|\mathcal{C}_1^k] + T\Delta_{opt}^{max}P(\widebar{\mathcal{F}}_2^k|\mathcal{C}_1^k) \notag \\
    \leq \mathbb{E}& [R_{\nu_2-\nu_1}^k\mathbb{I}(\mathcal{F}_2^k)|\mathcal{C}_1^k] +\mathbb{E}[R_{T-\nu_2}^k|\mathcal{C}_1^k] + T\Delta_{opt}^{max}M\sigma \notag \\
    \leq \tilde{C}_2^k &+ \Delta_2^*p(\nu_2-\nu_1) + \mathbb{E}[R_{T-\nu_2}^k|\mathcal{C}_1^k] + T\Delta_{opt}^{max}M\sigma. \notag 
\end{align}
Furthermore, we have
\begin{align}
    &\mathbb{E}[R_{T-\nu_2}^k|\mathcal{C}_1^k] \notag \\
    &\leq \mathbb{E}[R_{T-\nu_2}^k|\mathcal{F}_2^k \cap \mathcal{T}_2^k \cap \mathcal{C}_1^k] + T\Delta_{opt}^{max}(1-P(\mathcal{F}_2^k \cap \mathcal{T}_2^k|\mathcal{C}_1^k)) \notag \\
    & \leq \mathbb{E}[R_{T-\nu_2}^k|\mathcal{C}_2^k] + T\Delta_{opt}^{max}(M\sigma +\delta).
\end{align}
Wrapping up previous steps, we arrive at
\begin{align}
    \mathcal{R}_T^k &\leq \mathbb{E}[R_{T-\nu_2}^k|\mathcal{C}_2^k]+ \tilde{C}_1^k+\tilde{C}_2^k + \Delta_1^*p\nu_1 + \Delta_2^*p(\nu_2-\nu_1) \notag \\
    &\quad + T\Delta_{opt}^{max}(4M\sigma +2\delta).
\end{align}
Recursively, we can bound $\mathbb{E}[R_{T-\nu_2}^k|\mathcal{C}_{2}^k]$ by applying the same method as before. The regret upper bound then yields
\begin{gather}
    \mathcal{R}_T^k \leq \sum_{n=1}^{N}\tilde{C}_n^k + \Delta_{\text{opt}}^{\max}T(p+2MN\sigma + M\delta),
\end{gather}
where $\tilde{C}_n^k = \Delta_n^*[M\lceil \frac{8 \xi \log T}{(\Delta_n^{\min})^2} \rceil + M(1+\frac{\pi^2}{3})]$. 
\subsection{Proof of Corollary~\ref{cor:rgt}}\label{app:cor}
Based on Theorem~\ref{the:rgt}, the regret is upper bound by $\mathcal{R}_T^k \leq \sum_{n=1}^{N}\tilde{C}_n^k + \Delta_{opt}^{max}T(p+2MN\sigma + M\delta)$, let $\delta = \frac{1}{T}$ and $p = \sqrt{\frac{M\log T}{T}}$,
\begin{align}
    R_T &\leq \sum_{k=1}^K[\sum_{n=1}^{N}\tilde{C}_n^k + \Delta_{\text{opt}}^{\max}T (p+2MN\sigma + M\delta)] \notag \\
    &\leq KN \Delta_n^*[M\lceil \frac{8 \xi \log T}{(\Delta_n^{\min})^2} \rceil + M(1+\frac{\pi^2}{3})] \notag \\
    &\quad + K\Delta_{\text{opt}}^{\max} [\sqrt{MT\log T} + (2MN+M)].
\end{align}
\bibliographystyle{IEEEtran}
\bibliography{references}  
\end{document}